\documentclass[twoside]{article}

\usepackage{tikz}
\usetikzlibrary{positioning}

\usepackage[linesnumbered,ruled,vlined]{algorithm2e}
\usepackage{bm}
\usepackage{natbib}
\usepackage{graphicx}
\usepackage{comment}
\usepackage{enumerate}   
\usepackage{subcaption}
\usepackage{amssymb}
\usepackage{booktabs}
\usepackage{amsmath}
\newcommand{\ts}{\textstyle}
\newcommand{\bigCI}{\mathrel{\text{\scalebox{1.07}{$\perp\mkern-10mu\perp$}}}}

\DeclareMathOperator*{\argmax}{arg\,max}

\usepackage{latexsym}
\usepackage{amsmath}
\usepackage{appendix}
\newtheorem{theorem}{Theorem}
\newtheorem{corollary}{Corollary}
\newtheorem{assumption}{Assumption}

\newtheorem{example}{Example}[section]
\newtheorem{remark}{Remark}[section]

\newenvironment{proof}[1][Proof]{\textbf{#1.} }{\
\qquad\qquad\rule{0.5em}{0.5em }}
\newcommand{\bx}{\mathbf{x}}

\newcommand{\by}{\mathbf{y}}

\newcommand{\mI}{\mathcal{I}}
\newcommand{\tNA}{\text{NA}}

\newcommand{\mJ}{\mathcal{J}}
\newcommand{\op}{\mathrm{o}_{p}(n^{-1/2})}

\usepackage{color,graphicx,lscape,longtable}

\usepackage{times}

\usepackage[accepted]{aistats2020}
\begin{document}

%

%

\twocolumn[

\aistatstitle{Imputation Estimators for Unnormalized Models with Missing Data}

\aistatsauthor{Masatoshi Uehara \And Takeru Matsuda \And  Jae Kwang Kim}

\aistatsaddress{Harvard University \And The University of Tokyo \\ RIKEN Center for Brain Science \And Iowa State University } ]

\begin{abstract}
Several statistical models are given in the form of unnormalized densities, and calculation of the normalization constant is intractable. We propose estimation methods for such unnormalized models with missing data. 
The key concept is to combine imputation techniques with estimators for unnormalized models including noise contrastive estimation and score matching. 
In addition, we derive asymptotic distributions of the proposed estimators and construct confidence intervals. 
Simulation results with truncated Gaussian graphical models and the application to real data of wind direction reveal that the proposed methods effectively enable statistical inference with unnormalized models from missing data.
\end{abstract}

\section{INTRODUCTION}

Several statistical models are given in the form of unnormalized densities, and the calculation of the normalization constant (or partition function) is intractable. 
Namely, a statistical model is defined as
\begin{align}
    p(x ;\theta)=\frac{1}{Z(\theta)} \tilde{p}(x; \theta), \label{nnm}
\end{align}
where $Z(\theta) = \int \tilde{p}(x;\theta)\mu(\mathrm{d}x)$ is the normalization constant, $\mu$ is a base measure such as the Lebesgue measure or counting measure, and we only have access to $\tilde{p}(x;\theta)$. 
Such unnormalized models are widely used in many settings: Markov random fields \citep{BesagJulian1975SAoN}, directional statistics \citep{mardia99}, Boltzmann machines \citep{HintonGeoffreyE.2002TPoE}, overcomplete independent component analysis models \citep{HyvarinenAapo2001Ica}, and graphical models \citep{LinLina2016EoHG,Yu2016}. Several methods for estimating $\theta$ without computing the normalizing \\ \,\\   constant  $Z(\theta)$ have been proposed, such as noise contrastive estimation \citep[NCE;][]{noise}
and score matching \citep{score}.

In practice, we frequently encounter data with missing values, which is called missing data or incomplete data \citep{TsiatisAnastasiosA.AnastasiosAthanasios2006Stam,kimshao13}.
Missing data must be handled properly; otherwise incorrect estimates may be obtained such as nonresponse bias \citep{LittleRoderickJ.A2002Sawm}.
However, existing estimation methods for unnormalized models are not applicable to missing data, because they assume that the data is fully observed.

In this study, we develop estimation methods for unnormalized models with missing data. 
The proposed methods utilize NCE and score matching by imputing missing data with importance weights. This method is computationally fast because it does not rely on any sampling techniques. We derive asymptotic distributions of the proposed estimators and construct confidence intervals. On the way, we also discuss how to incorporate multiple imputation \citep{meng94} and contrastive divergence method \citep{HintonGeoffreyE.2002TPoE}. 

Note that \cite{Rhodes} proposed an estimation method called variational NCE for unnormalized latent variable models, which corresponds to a special case of the current problem (missing completely at random, MCAR). 
Although variational inference is efficient and useful for large--scale problems, it is not clear how to construct its confidence intervals \citep{BleiDavidM.2017VIAR}. 
In contrast, the proposed methods are valid under general missing mechanisms, including missing at random (MAR) and missing not at random (MNAR) cases.
In addition, the proposed methods allow for the construction of confidence intervals based on asymptotic theory.

Our main contributions are as follows. 

\begin{itemize}
    \item  We propose imputation estimators for unnormalized models with missing data. These estimators are $\sqrt{n}$--consistent under general missing mechanisms, including MNAR, and computationally efficient.
    
    \item  We derive asymptotic distributions of the proposed estimators and construct confidence intervals.
    
    \item We confirm the validity of the proposed methods by simulation with truncated Gaussian graphical models and apply the proposed methods to analyze real data of wind direction with the bivariate circular model. 
\end{itemize}

\section{PRELIMINARY}

\subsection{Notations}
Parameters with a zero in the subscript such as $\theta_0$ and $\tau_0$ denote true parameters. 
The notation $\nabla_{\theta}$ denotes differentiation with respect to $\theta$, and $t(x)^{\otimes 2}=t(x)t(x)^{\top}$. The notation $\stackrel{d}{\rightarrow}$ denotes a weak convergence. The expectation and variance of $f(x)$ under density $g(x)$ are denoted as $\mathrm{E}_{g}[f(x)]$ and $\mathrm{var}_{g}[f(x)]$, respectively. 
The subscript and argument are often omitted when they are clear from the context. A summary of the notation is provided in Appendix \ref{sec:notation}.  

\subsection{Estimation Methods for Unnormalized Models}

Several methods have been developed for estimating the unnormalized model \eqref{nnm} such as noise contrastive estimation \citep[NCE;][]{noise}, score matching \citep{score}, and Monte Carlo maximum likelihood estimation \citep[MC--MLE;][]{GeyerC1994Otco}.

\subsubsection{Noise Contrastive Estimation (NCE)}

{In NCE, the unnormalized model \eqref{nnm} is rewritten to a one-parameter extended model $q(x;\tau)=\exp(-c)\tilde{p}(x;\theta)$, where $\tau=(c,\theta)$ and the true value of $c$ is $c=\log Z(\theta)$.}

{In addition to data samples $\bx=\{x_i\}_{i=1}^{n}$ from the unnormalized model \eqref{nnm}, we generate noise samples $\by=\{y_j\}_{j=1}^{n'}$ from a noise distribution with density $a(y)$.
Just for simplicity, we set $n'=n$ in the following.}

Let 
be the density ratio.
Then, in NCE, the estimator is defined as the maximizer of the following function  with respect to $\tau$; 
\begin{align}\ts
   \sum_{i=1}^{n}\log \frac{r(x_i;\tau)}{r(x_i;\tau)+ 1} + \sum_{j=1}^{n}\log \frac{1}{r(y_j;\tau)+1}. \label{nce_def}
\end{align}
This objective function is interpreted as the negative log-likelihood of the naive Bayes classifier. Regarding more intuitive explanations, see \cite{noise}. NCE gives a $\sqrt{n}$--consistent estimator under mild regularity conditions. 



NCE can be generalized from the divergence perspective \citep{noise2, hirayama}. 
Let $g(x)$ be the true data distribution and consider the Bregman divergence
\begin{align*}\ts
D_f(g(x),q(x;\tau))= \int \mathrm{Br}_{f}\left (\frac{g(x)}{a(x)},\frac{q(x;\tau)}{a(x)}\right)a(x)\mu(\mathrm{d}x),
\end{align*}
where $f$ is a twice differentiable strictly convex function and  $\mathrm{Br}_{f}(u,v)=f(u)-f(v)-f'(v)(u-v)$. 
By subtracting a term independent of $\tau$ from $D_f(g,q(x;\tau))$, the cross entropy $d_{f}(g(x),q(x;\tau))$ between $g(x)$ and $q(x;\tau)$ is obtained as 
\begin{align*}\ts
\mathrm{E}_{g(x)}\left[m_{nc1}(x;\tau)\right]+ \mathrm{E}_{a(y)}\left[m_{nc2}(y;\tau)\right],
\end{align*}
where 
\begin{align*}\ts
m_{nc1}(x;\tau) &=-f'(r(x;\tau)),\\ m_{nc2}(y;\tau) &=f'\left(r(y;\tau)\right )r(y;\tau)-f(r(y;\tau)).
\end{align*}
Then, the generalized NCE is defined as the minimizer of $M_{nc1}(\bx;\tau)+M_{nc2}(\by;\tau)$ with respect to $\tau$, where $M_{nc1}(\bx) =n^{-1}\sum_{i=1}^{n}m_{nc1}(x_i;\tau)$ and $M_{nc2}(\by) =n^{-1}\sum_{j=1}^{n}m_{nc2}(y_{j};\tau)$.
By differentiation with respect to $\tau$, the estimator is also given by the solution to $Z_{nc1}(\bx;\tau)+Z_{nc2}(\by;\tau)=0$, where $Z_{nc1}(\bx;\tau)=n^{-1}\sum_{i=1}^{n}z_{nc1}(x_i;\tau),\,   Z_{nc2}(\by;\tau)=n^{-1}\sum_{j=1}^{n}z_{nc2}(y_j;\tau)$
and 
\begin{align*}\ts
    z_{nc1}(x;\tau)&=-\nabla_{\tau}\log q(x;\tau)f''\left(r(x;\tau)\right)r(x;\tau),\\
    z_{nc2}(y;\tau)&=\nabla_{\tau}\log q(y;\tau)f''\left(r(y;\tau)\right)r(y;\tau)^{2}.
\end{align*}
The original NCE \eqref{nce_def} corresponds to $f(x)=x\log x-(1+x)\log(1+x)$ and it is optimal in terms of asymptotic variance \citep{Uehara}. 
On the other hand, when $f(x)=x\log x$, the estimator is given by the minimizer of
\begin{align*}\ts
    -\frac{1}{n}\sum_{i=1}^{n}\log q(x_{i};\tau)+\frac{1}{n}\sum_{j=1}^{n}r(y_{j};\tau), 
\end{align*}
which is essentially identical to MC--MLE \citep{GeyerC1994Otco} by profiling-out $c$.

\subsubsection{Score Matching}

Score matching was originally developed as a general estimation method for the unnormalized model \eqref{nnm} on $\mathbb{R}^d$. 
Let $c(x;\theta)=\nabla_{x} \log \tilde{p}(x;\theta) \in \mathbb{R}^d$ and denote the $s$-th coordinate of $x$ by $x^{s}$. 
For data samples $\bx=\{x_i\}_{i=1}^{n}$, the score matching estimator of $\theta$ is defined as the minimizer of $M_{sc}(\bx;\theta)=n^{-1}\sum_{i=1}^{n}m_{sc}(x_i;\theta)$, where
\begin{align*}\ts
m_{sc}(x;\theta)&=\frac{1}{2} \sum_{s=1}^d c_s(x;\theta)^2+\sum_{s=1}^{d}\frac{\partial c_{s}(x;\theta)}{\partial x^s},\\
c_s(x;\theta)&=\frac{\partial}{\partial x^s}\log \tilde{p}(x;\theta).
\end{align*}
Note that this estimator is also given by the solution to $Z_{sc}(\bx;\theta)=0$, where $ Z_{sc}(\bx;\theta)=n^{-1}\sum_{i=1}^{n}z_{sc}(x_i;\theta)$ and $z_{sc}(x;\theta)=\nabla_{\theta}m_{sc}(x;\theta)$. 


\cite{HyvarinenAapo2007Seos} extended score matching to unnormalized models on  $\mathbb{R}_+^d = [0,\infty)^d$. 
The estimator is defined as the minimizer of $M_{sc+}(\bx; \theta)=n^{-1}\sum_{i=1}^{n}m_{sc+}(x_i;\theta)$, where $m_{sc+}(x;\theta)$ is given by
\begin{align*}
\sum_{s=1}^{d} \left \{2x^{s}c_{s}(x;\theta)+(x^s)^{2}c_{s}(x;\theta)^{2}+ (x^s)^2 \frac{\partial c_{s}(x;\theta)}{\partial x^s}\right\}.
\end{align*}

\subsection{Missing Data and Imputation Methods}

We briefly review the framework of missing data and imputation methods. For more details, see \cite{kimshao13}. 

Suppose that $\{x_i\}_{i=1}^{n}$ are independently and identically distributed (i.i.d.) samples from a distribution with density $p(x;\theta)$.
We consider the situation where some part of $x_i$ may be missing. Let $\{\delta_i\}_{i=1}^{n}$ be the missing indicators.
Accordingly, $x_i=(x_{i,\mathrm{obs}},x_{i,\mathrm{mis}})$ is fully observed when $\delta_i=1$, while only $x_{i,\mathrm{obs}}$ is observed and $x_{i,\mathrm{mis}}$ is missing when $\delta_i=0$. 
We assume that $\delta_i$ follows a Bernoulli distribution with probability ${\rm Pr}(\delta_i=1 \mid x_i)$.  
The case with several missing patterns, that is, the case where the dimension of $x_{i,\mathrm{obs}}$ may differ with $i$, can be easily considered by extending this notation \citep{SeamanShaun2013WIMb}. For more details, see Appendix \ref{sec:multiple}.

The missing mechanism is called missing at random (MAR) if ${\rm Pr}(\delta=1 \mid x) = {\rm Pr}(\delta=1 \mid x_{\mathrm{obs}})$ holds. 
Importantly, the missing process can be ignored for estimation of $\theta$ in MAR cases  \citep{LittleRoderickJ.A2002Sawm}, because 
\begin{align*}\ts
    p(x_{\mathrm{obs}};\theta) &= \int p(x_{\mathrm{obs}},x_{\mathrm{mis}};\theta)\mathrm{Pr}(\delta \mid x)\mu(\mathrm{d}x_{\mathrm{mis}}) \\
    &\propto \int p(x_{\mathrm{obs}},x_{\mathrm{mis}};\theta)\mu(\mathrm{d}x_{\mathrm{mis}}).
\end{align*}  
As a special case of MAR, a missing mechanism is referred to as missing completely at random (MCAR) if ${\rm Pr}(\delta=1 \mid x)$ is independent of $x$. When MAR does not hold, the missing mechanism is referred to as missing not at random (MNAR).  

For estimating $\theta$ from missing data, the fundamental algorithm is the expectation maximization (EM) algorithm \citep{dempster77}, which maximizes the observed likelihood $p(x_{\mathrm{obs}};\theta)$. 
Equivalently, the EM algorithm solves the following observed (mean) score equation with respect to $\theta$ \citep{louis82}: 
\begin{align}\ts
\label{eq:score2}
\frac{1}{n}\sum_{i=1}^{n}\mathrm{E}\left[\nabla_{\theta}\log p(x_i;\theta) \mid x_{i,\mathrm{obs}};\theta\right]=0.
\end{align}
However, the EM algorithm requires a closed-form expression of the conditional expectation in \eqref{eq:score2}, which is often intractable.
To overcome this obstacle, a method called fractional imputation (FI) has been proposed \citep{Kim11,YangShu2016FIiS}, which is closely connected with the Monte Carlo EM algorithm \citep{WeiGregC.G.1990AMCI}. 
FI is computationally efficient because it uses importance sampling and does not rely on MCMC. 
Another method called multiple imputation (MI) is also commonly used, which utilizes MCMC \citep{rubin87,MurrayJaredS.2018MIAR}.

\section{FINCE and FISCORE}
We propose two estimation methods for unnormalized models with missing data: FINCE (fractional imputation noise contrastive estimation) and FISCORE (fractional imputation score matching).

In this section, we focus on the MAR case, that is, $\mathrm{Pr}(\delta=1 \mid x)= \mathrm{Pr}(\delta=1 \mid x_{\mathrm{obs}})$. In Section \ref{sec:extensions}, we discuss an extension to the MNAR case.

Throughout this section, we assume one missing pattern.
For the case of multiple missing patterns, see Appendix \ref{sec:multiple}.

\subsection{NCE with EM algorithm}
First, we incorporate the EM algorithm into NCE. 
Although the score equation cannot be used as in \eqref{eq:score2} for unnormalized models, the estimating equation $Z_{nc1}(\bx;\tau)+Z_{nc2}(\by;\tau)=0$ of NCE can be used instead. 
Thus, the estimator of ${\tau}=({c},{\theta})$ is defined as the solution to the following equation:
\begin{align}
\label{eq:em}
     \frac{1}{n}\sum_{i=1}^{n}\mathrm{E}[z_{nc1}(x_i;\tau)|x_{i,\mathrm{obs}};\theta]+
   \frac{1}{n} \sum_{j=1}^{n}z_{nc2}(y_j;\tau)=0,
\end{align}
where each conditional expectation in the first term is taken with respect to the posterior
\begin{align}
\label{eq:mcmc}
   p(x_{i,\mathrm{mis}} \mid x_{i,\mathrm{obs}};\theta)=\frac{\tilde{p}(x_i;\theta)}{\int \tilde{p}(x_i;\theta)\mu(\mathrm{d}x_{i,\mathrm{mis}})}.
\end{align}
Note that the first term in the left hand side of \eqref{eq:em} formally means 
\begin{align}
\label{eq:em2}
\frac{1}{n}\sum_{i=1}^{n}\left \{\delta_i z_{nc1}(x_i;\tau)+(1-\delta_i)\mathrm{E}[z_{nc1}(x_i;\tau) \mid x_{i,\mathrm{obs}};\theta]\right\},   
\end{align}
because the dimension of $x_{i,\mathrm{obs}}$ may vary with $i$. 
Throughout this paper, we implicitly assume this conversion following the convention in the literature of missing data \citep{SeamanShaun2013WIMb}. 

Generally, it is difficult to analytically calculate the conditional expectation in \eqref{eq:em}. 
In the subsequent subsection, we develop a method to resolve this problem. 
Here, assuming that the conditional expectation in \eqref{eq:em} can be calculated analytically, we propose the EM algorithm to solve the equation \eqref{eq:em}, which is given by Algorithm \ref{al:em}. 
\begin{algorithm}
\SetKwInOut{Input}{input}
\SetKwInOut{Output}{output}
\Input{$\{x_i\}_{i=1}^{n}$,\, $\hat{\tau}_0=(\hat{c}_0,\hat{\theta}_0)$}
\Output{$\hat{\tau}_T=(\hat{c}_T,\hat{\theta}_T)$}
Initialize $t=0$ \\ 
Take $n$ samples $\{y_j\}_{j=1}^{n}$ from $a(y)$  \\
   \Repeat{$\hat{\tau}_{t}$ converges  }{
        Solve the following equation w.r.t $\tau$ and update $\hat{\tau}_{t+1}$: 
     \begin{align*}
     \mathrm{E}[Z_{nc1}(\bx;{\tau}) \mid \bx_{\mathrm{obs}};{\theta}_{t}]+Z_{nc2}(\by;{\tau})=0. 
      \end{align*} \\
           $t=t+1$}
\caption{NCE with EM algorithm }
\label{al:em}
\end{algorithm}

Note that each update of $\hat{\tau}_t$ in Algorithm \ref{al:em} can be replaced with M-estimators. 
For example, when $f(x)=x\log x-(1+x)\log(1+x)$ (original NCE), $\hat{\tau}_{t+1}$ is obtained by maximization with respect to $\tau$ of 
\begin{align}
\label{eq:nce}
\sum_{i=1}^{n}\mathrm{E}\left[\left.\log \frac{r(x_i;\tau)}{r(x_i;\tau)+ 1} \right| x_{i,\mathrm{obs}};\hat{\theta}_{t}\right]+\sum_{j=1}^{n}\log\frac{1}{r(y_j;\tau)+ 1}. 
\end{align}
When $f(x)=x\log x$ (MC--MLE), $\hat{\tau}_{t+1}$ is obtained by minimization with respect to $\tau$ of
 \begin{align*}
-\frac{1}{n}\sum_{i=1}^{n}\mathrm{E}[\log q(x_i;\tau) \mid x_{i,\mathrm{obs}};\hat{\theta}_{t}]+ 
   \frac{1}{n} \sum_{j=1}^{n} r(y_{j};\tau). 
\end{align*}

\begin{remark}[Difference from variational NCE]
From \eqref{eq:nce}, Algorithm \ref{al:em} and variational NCE \citep{Rhodes} are different. 
See Appendix \ref{sec:difference} for details.  
\end{remark}

\subsection{NCE with Fractional Imputation (FINCE)}
\label{sec:fince}

It is often infeasible to analytically calculate the conditional expectation in Algorithm \ref{al:em}. 
Thus, in the same spirit of fractional imputation \citep[FI;][]{Kim11}, we incorporate importance sampling using random variables from an auxiliary distribution $b(x_{\mathrm{mis}})$.
Namely, we use the formula
\begin{align*}
    &\int u(x)p(x_{\mathrm{mis}} \mid x_{\mathrm{obs}};\theta)\mu(\mathrm{d}x_{\mathrm{mis}})\\
    &= \frac{\mathrm{E}_{b(x_{\mathrm{mis}})} [u(x) \tilde{p}(x_{\mathrm{mis}},x_{\mathrm{obs}};\theta)/b(x_{\mathrm{mis}})]}{\mathrm{E}_{b(x_{\mathrm{mis}})}[\tilde{p}(x_{\mathrm{mis}},x_{\mathrm{obs}};\theta)/b(x_{\mathrm{mis}})]}
\end{align*}
to calculate $\mathrm{E}[Z_{nc1}(\bx;\tau) \mid \bx_{\mathrm{obs}};\theta]$ in \eqref{eq:em}. 
The resulting procedure is given in Algorithm \ref{al:em3}. 
Here, $\propto$ in the W-step indicates a normalization so that the summation of $w_{ik}$ over $k$ is equal to $1$ for each $i$. 

\begin{algorithm}[h]
\SetKwInOut{Input}{input}
\SetKwInOut{Output}{output}
\Input{$\{x_i\}_{i=1}^{n}$,\, $\hat{\tau}_0=(\hat{c}_0,\hat{\theta}_0)$}
\Output{$\hat{\tau}_T=(\hat{c}_T,\hat{\theta}_T)$}
Initialize $t=0$ \\ 
Take $n$ samples $\{y_{j}\}_{j=1}^{n}$ from $a(y)$.\\
For $i$ with $\delta_i=0$, take $m$ samples $\{x_{i,\mathrm{mis}}^{*k}\}_{k=1}^{m}$ from $b(x_{\mathrm{mis}})$ and set $x_{i}^{*k}=(x_{i,\mathrm{obs}},x^{*k}_{i,\mathrm{mis}})$. \\
For $i$ with $\delta_i=1$, set $m$ samples $\{x_{i}^{*k}\}_{k=1}^{m}$ to $x_{i}^{*k}=x_{i}$
\\
   \Repeat{$\hat{\tau}_{t}$ converges}
   {
   		W-Step:    \\ 
   		For $i$ with $\delta_i=0$;
   		$w_{ik} \propto q(x_{i}^{*k};\hat{\tau}_{t})/b(x^{*k}_{i,\mathrm{mis}})$.
    This means 
        \begin{align*}
         w_{ik}=\frac{q(x_{i}^{*k};\hat{\tau}_t)/b(x^{*k}_{i,\mathrm{mis}})}{\sum_{k=1}^{m}q(x_{i}^{*k};\hat{\tau}_t)/b(x^{*k}_{i,\mathrm{mis}})}. 
        \end{align*} \\
          For $i$ with $\delta_i=1$; $w_{ik}=1/m$
        \\
        M-step: Solve the following equation w.r.t $\tau$ and update $\hat{\tau}_{t+1}$:
        \begin{align*}
       \left[ \frac{1}{n}\sum_{i=1}^{n}\sum_{k=1}^{m}w_{ik}z_{nc1}(x_{i}^{*k};\tau)\right]+Z_{nc2}(\by;\tau)=0.
         \end{align*}
         $t = t+1$
   }
\caption{FINCE }
\label{al:em3}
\end{algorithm}

Note that again the update of $\hat{\tau}_t$ in M-step can be replaced with M-estimators. 
For example, the conditional expectation in \eqref{eq:nce} is calculated as 
\begin{align*}
 \sum_{k=1}^{m} w_{ik}\log \frac{r(x_{i}^{*k};\tau)}{r(x_{i}^{*k};\tau)+1}.
\end{align*}

The choice of the noise and auxiliary distributions is important for improved estimation accuracy. 
Specifically, the noise distribution $a(x)$ should be generally close to $p(x_{\mathrm{mis}},x_{\mathrm{obs}};\theta_{0})$, while the auxiliary distribution $b(x_{\mathrm{mis}})$ should be close to $p(x_{\mathrm{mis}}\mid x_{\mathrm{obs}};\theta_{0})$.
When there are samples without missing data (complete data) as in Section \ref{sec:experiment}, moment matching with complete data can be used to determine $a(x)$ and $b(x_{\mathrm{mis}})$. 
We recommend using heavy-tailed distributions for $a(x)$ and $b(x_{\mathrm{mis}})$, following the common strategy of importance sampling \citep{mcbook}.

\subsection{Score Matching with Fractional Imputation (FISCORE)}

Since score matching is defined in the form of Z-estimators like NCE, we can define score matching with the EM algorithm as the solution to $Z_{sc,\mathrm{obs}}(\bx_{\mathrm{obs}};\theta)=0$, where
\begin{align}
\label{eq:fiscore}
Z_{sc,\mathrm{obs}}(\bx_{\mathrm{obs}};\theta)=\mathrm{E}[Z_{sc}(\bx;\theta) \mid \bx_{\mathrm{obs}}; \theta].
\end{align}
Since calculation of the conditional expectation in \eqref{eq:fiscore} is often challenging, we again propose using importance sampling with an auxiliary distribution $b(x)$ like FINCE.
The resulting procedure of FISCORE is provided in Algorithm \ref{al:em4}. 
A similar algorithm is obtained for the non--negative score matching. 

\begin{algorithm}
\SetKwInOut{Input}{input}
\SetKwInOut{Output}{output}
\Input{$\{x_i\}_{i=1}^{n}$,\,$\hat{c}_0$,\,$\hat{\theta}_0$}
\Output{$\hat \theta_{T}$}
Initialize $t=0$, $\hat{\tau}_{0}=(\hat{c}_0,\hat{\theta}_0)$. \\
For each $i$ with $\delta_i=0$, take $m$ samples $\{x_{i,\mathrm{mis}}^{*k}\}_{k=1}^{m}$ from $b(x)$.\\
For $i$ with $\delta_i=1$, set $m$ samples $\{x_{i}^{*k}\}_{k=1}^{m}$ so that $x_{i}^{*k}=x_{i}$\\
   \Repeat{$\hat{\tau}_{t}$ converges }
   {
   		W-Step: \\
   		If $\delta_i=0$; 
   		$w_{ik} \propto \tilde{p}(x_{i}^{*k};\hat{\theta}_{t})/b(x^{*k}_{\mathrm{mis}})$. \\
   		If $\delta_i=1$; $w_{ik}=1/m$ \\
       M-step: Solve the following equation to obtain for $\hat{\theta}_{t+1}$ w.r.t $\theta$: 
        \begin{align*}
            \frac{1}{n} \sum_{i=1}^{n}\sum_{k=1}^{m}w_{ik}z_{sc}(x_{i}^{*k};{\theta}) = 0. 
        \end{align*}\\
        $t=t+1$
   }
\caption{FISCORE }
\label{al:em4}
\end{algorithm}

\begin{remark}[MINCE and MISCORE]
We can also combine MI with NCE and score matching though it is unstable and computationally heavy. 
See Appendix \ref{sec:MI}.  
\end{remark}

\begin{remark}[FICD]
We can also extend our approach to the contrastive divergence (CD) method as in Appendix \ref{sec:CD}.
\end{remark}

\section{ASYMPTOTICS AND CONFIDENCE INTERVALS}
\label{sec:theory}

We derive the asymptotic distributions of FINCE and FISCORE. 
Based on the asymptotic distributions, we construct confidence intervals, which enable hypothesis testing.
This is an advantage of the proposed methods over variational NCE \citep{Rhodes}. 
Proofs of the theorems are given in the Appendix. 

\subsection{FINCE}

We start with the analysis of FINCE with the EM algorithm $\hat{\tau}_{nc}$, which is the solution to $Z_{nc,\mathrm{obs}}(\mathbf{x}_{\mathrm{obs}},\by;\tau)=0$ where
\begin{align*}
Z_{nc,\mathrm{obs}}(\bx_{\mathrm{obs}},\by;\tau)=\mathrm{E}[Z_{nc1}(\bx;\tau) \mid \bx_{\mathrm{obs}}; \tau]+Z_{nc2}(\by;\tau).
\end{align*}
Based on the theory of Z-estimators \citep{VaartA.W.vander1998As}, its asymptotic distribution is obtained as follows. 

\begin{theorem}\label{thm:easy2}

We have 
\begin{align*}
    \sqrt{n} (\hat{\tau}_{nc}-\tau_0) \stackrel{d}{\rightarrow}\mathrm{N} (0, \mI_{1,nc}^{-1}\mJ_{1,nc}(\mI_{1,nc}^{\top})^{-1}),
\end{align*}
where
\begin{align*}
\mI_{1,nc}=\mathrm{E}[\nabla_{\tau^{\top}}Z_{nc,\mathrm{obs}}(\bx_{\mathrm{obs}},\by;\tau_0)],\\
\mJ_{1,nc}=\mathrm{var}[Z_{nc,\mathrm{obs}}(\bx_{\mathrm{obs}},\by;\tau_0)].     
\end{align*}
\end{theorem}

Next, we investigate FINCE by considering each iteration. 
Given an initial $\sqrt{n}$-consistent estimator $\hat{\tau}_{p}$, we obtain the imputed equation $Z_{nc,m}(\tau \mid \hat{\tau}_{p})=0$, where $Z_{nc,m}(\tau \mid \hat{\tau}_{p})$ is given by
\begin{align*}
\left \{\frac{1}{n}\sum_{i=1}^{n}\sum_{k=1}^{m}w(x_{i}^{*k};\hat{\tau}_{p})z_{nc1}(x_{i}^{*k};\tau)\right \}+\frac{1}{n}\sum_{j=1}^{n}z_{nc2}(y_{j};\tau),
\end{align*}
where $x^{*k}_{i}=(x_{i,\mathrm{obs}},x^{*k}_{i,\mathrm{mis}})$, $x^{*k}_{i,\mathrm{mis}}\sim b(x_{\mathrm{mis}})$, and $w(x_{i}^{*k};\hat{\tau}_p)\propto{q(x_{i}^{*k};\hat{\tau}_p)/b(x_{i,\mathrm{mis}}^{*k}) }$.
As $m \to \infty$, $Z_{nc,m}(\tau \mid \hat{\tau}_{p})$ converges to $\bar{Z}_{nc}(\tau \mid \hat{\tau}_{p})$ given by
\begin{align*}
\left \{\frac{1}{n}\sum_{i=1}^{n}\mathrm{E}[z_{nc1}(x_i;\tau) \mid x_{i,\mathrm{obs}};\hat{\tau}_{p}]\right \}+\frac{1}{n}\sum_{j=1}^{n}z_{nc2}(y_{j};\tau).
\end{align*}
Let $\hat{\tau}_{nc,\infty}$ be the solution to $\bar{Z}_{nc}(\tau \mid \hat{\tau}_{p})=0$.
Then, as proved later in Theorem \ref{thm:main}, we obtain
\begin{align*}
\hat{\tau}_{nc,\infty}=\hat{\tau}_{nc} +\mathcal{I}^{-1}_{3,nc}\mathcal{I}_{2,nc}(\hat{\tau}_{p}-\hat{\tau}_{nc})+\mathrm{o}_{p}(n^{-1/2}),
\end{align*}
where $\mathcal{I}_{3,nc}=\mathrm{E}[\nabla_{\tau^{\top}}Z_{nc}(\bx,\by;\tau_0)]$ and $\mathcal{I}_{2,nc}=\mathcal{I}_{3,nc}-\mathcal{I}_{1,nc}$. 

Let $\hat{\tau}^{(0)}=\hat{\tau}_p$ and define $\hat{\tau}^{(t)}$ to be the solution to $\bar{Z}_{nc}(\tau \mid \hat{\tau}^{(t-1)})=0$ for each $t$.
Then, we obtain the following. 
\begin{corollary} 
\label{cor:clear}
We have
\begin{align*}
  \hat{\tau}^{(t)}=\hat{\tau}_{nc}+(\mI_{3,nc}^{-1}\mI_{2,nc})^{t-1}(\hat{\tau}^{(0)}-\hat{\tau}_{nc})+\mathrm{o}_{p}(n^{-1/2}).
\end{align*}
If the spectral radius of $\mI_{3,nc}^{-1}\mI_{2,nc}$ is less than $1$, then $\hat{\tau}^{(t)}$ converges to $\hat{\tau}_{nc}$ as $t \to \infty$. 
\end{corollary}

Let $v(x;\tau)=\nabla_{\tau}\log q(x;\tau)$ and $r_0(x)=r(x;\tau_0)=q(x;\tau_0)/a(x)$.
For the original NCE, each term in the above is explicitly obtained as follows. 

\begin{corollary}
\label{cor:nce}
When $f(x)=x\log x-(1+x)\log(1+x)$, 
\begin{align*}
\mI_{1,nc}=&\mathrm{E} \left[\mathrm{E}\left[\left.\frac{v(x;\tau_0)}{1+r_0(x)}\right|x_{\mathrm{obs}}\right]\mathrm{E}\left[v(x;\tau_0)^{\top} \mid x_{\mathrm{obs}}\right]\right],\\
\mI_{3,nc}=&\mathrm{E}\left[\frac{v(x;\tau_0)^{\otimes 2}}{1+r_0(x)}\right],\\
\mJ_{1,nc}=&\mathrm{var}_{q}[\mathrm{E}[z_{nc1}(x;\tau_0) \mid x_{\mathrm{obs}}]]+\mathrm{var}_{a}[z_{nc2}(y;\tau_0)], \\
z_{nc1}(x;\tau)&=-\frac{v(x;\tau)}{1+r(x;\tau)},\quad z_{nc2}(y;\tau)=\frac{r(y;\tau) v(y;\tau)}{1+r(y;\tau)}.
\end{align*}
\end{corollary}

For MC--MLE, we can prove the convergence of FINCE as follows.

\begin{corollary}

\label{cor:convergece}
When $f(x)=x\log x$, 
\begin{align*}
    \mI_{1,nc}=\mathrm{E}\left[\mathrm{E}\left[v(x;\tau_0)|x_{\mathrm{obs}}\right]^{\otimes 2}\right],\,
    \mI_{3,nc}=\mathrm{E}\left[v(x;\tau_0)^{\otimes 2}\right].
\end{align*}
Additionally, $\{\mI_{3,nc}^{-1}\mI_{2,nc}\}^{j}\to 0$ as $j\to \infty$. 
\end{corollary}



\subsection{FISCORE}
\label{sec:the_fiscore}

First, we analyze score matching with the EM algorithm $\hat{\theta}_{sc}$, which is the solution of $Z_{sc,\mathrm{obs}}(\bx_{\mathrm{obs}};\theta)=0$ where $Z_{sc,\mathrm{obs}}(\bx_{\mathrm{obs}};\theta)$ is defined by \eqref{eq:fiscore}. 
The asymptotic distribution of $\hat{\theta}_{sc}$ is obtained as follows.
\begin{theorem}
\label{thm:easy}

We have 
\begin{align*}
    \sqrt{n} (\hat{\theta}_{sc}-\theta_0) \stackrel{d}{\rightarrow} \mathrm{N} (0, \mI_{1,sc}^{-1}\mJ_{1,sc}(\mI_{1,sc}^{\top})^{-1}),
\end{align*}
where
\begin{align*}
    \mI_{1,sc} &= \mathrm{E}[\nabla_{\theta^{\top}}Z_{sc,\mathrm{obs}}(\bx_{\mathrm{obs}};\theta_{0})],\\
    \mJ_{1,sc} &= \mathrm{var}[Z_{sc,\mathrm{obs}}(\bx_{\mathrm{obs}};\theta_{0})]. 
\end{align*}

\end{theorem}

Next, we study the asymptotic property of FISCORE by focusing on each update. 
Given an initial $\sqrt{n}$-consistent estimator $\hat{\theta}_{p}$ for $\theta$, consider the imputed equation:
\begin{align*}
    Z_{sc,m}(\theta \mid \hat{\theta}_{p}):=\frac{1}{n}\sum_{i=1}^{n}\sum_{k=1}^{m}w(x_{i}^{*k};\hat{\theta}_{p})z_{sc}(\theta;x_{i}^{*k})=0,
\end{align*}
where $x^{*k}_{i}=(x_{i,\mathrm{obs}},x^{*k}_{i,\mathrm{mis}})$, $x^{*k}_{i,\mathrm{mis}}\sim b(x_{\mathrm{mis}})$, and $w(x_{i}^{*k};\hat{\theta}_p)\propto{\tilde{p}(x_{i}^{*k};\hat{\theta}_p)/b(x_{i,\mathrm{mis}}^{*k}) }$.

Here, we consider the case $m \to \infty$.
See Appendix \ref{sec:finite} for the case of finite $m$. 
As $m \to \infty$, the function
$Z_{sc,m}(\theta \mid \hat{\theta}^{p})$ converges to 
\begin{align*}
    \bar{Z}_{sc}(\theta \mid \hat{\theta}_{p})=\mathrm{E}[Z_{sc}(\bx;\theta) \mid \mathbf{x}_{\mathrm{obs}};\hat{\theta}_{p}].
\end{align*}
Let $\hat{\theta}_{sc,\infty}$ be the solution to $\bar{Z}_{sc}(\theta \mid \hat{\theta}_{p})=0$.
Its asymptotic property is obtained as follows.

\begin{theorem}
\label{thm:main}
We have 
\begin{align*}
\hat{\theta}_{sc,\infty}=\hat{\theta}_{sc} +\mathcal{I}^{-1}_{3,sc}\mathcal{I}_{2,sc}(\hat{\theta}_{p}-\hat{\theta}_{sc})+\mathrm{o}_{p}(n^{-1/2}),\,
\end{align*}
where 
\begin{align*}
\mathcal{I}_{2,sc}&=-\mathrm{E}[\mathrm{cov}[z_{sc}(x;\theta_{0}),\nabla_{\theta}\log \tilde{p}(x;\theta_{0}) \mid x_{\mathrm{obs}} ]], \\
    \mathcal{I}_{3,sc}&=\mathrm{E}\left[\sum_{s=1}^{d}\nabla_{\theta}c_s(x;\theta_0)^{\otimes 2}\right]. 
\end{align*}
\end{theorem}

In the proof of Theorem \ref{thm:main}, we use the relation $\mI_{3,sc}=\mI_{1,sc}+\mI_{2,sc}$, which corresponds to the missing information principle or Louis' formula for normalized models \citep{kimshao13,orchard72,louis82}. 
Specifically, if $Z_{sc}(\bx;\theta)$ is replaced by the true score function $S_{sc}(\mathbf{x};\theta)=\nabla_{\theta}\log p(\mathbf{x};\theta)$, then Theorem~\ref{thm:main} reduces to the result of \cite{wang98}. 
In this case, $\mathcal{I}_{3,sc}$,\,$\mathcal{I}_{1,sc}$ and $\mathcal{I}_{2,sc}$ are replaced by
\begin{align*}
    \mathcal{I}_{\mathrm{com}}&=\mathrm{E}[\nabla_{\theta^{\top}}S_{sc}(\bx;\theta_{0})],
\,    \mathcal{I}_{\mathrm{obs}} =\mathrm{E}[\nabla_{\theta^{\top}}S_{\mathrm{obs}}(\bx_{\mathrm{obs}};\theta_{0})],\\
    \mathcal{I}_{\mathrm{mis}}&=\mathrm{E}[S_{\mathrm{mis}}(\bx;\theta_{0})^{\otimes 2}],
\end{align*}
respectively, where
\begin{align*}
    S_{\mathrm{mis}}(\bx;\theta)&=S_{sc}(\bx;\theta)-\mathrm{E}[S_{sc}(\bx;\theta) \mid \mathbf{x}_{\mathrm{obs}};\theta], \\
    S_{\mathrm{obs}}(\bx_{\mathrm{obs}};\theta)&=\int S_{sc}(\bx;\theta)\mu(\mathrm{d}\bx_{\mathrm{mis}}).
\end{align*}
The relation $\mI_{\mathrm{com}}=\mI_{\mathrm{obs}}+\mI_{\mathrm{mis}}$ holds, and the term $\mathcal{I}_{\mathrm{com}}^{-1}\mathcal{I}_{\mathrm{mis}}$ is often called the fraction of missing information \citep{kimshao13}. For the current problem, $\mathcal{I}_{3,sc}^{-1}\mathcal{I}_{2,sc}$ can be considered as an analog. As seen in Corollary \ref{cor:clear}, this qunantity is important to guarantee the convergence. It is generally difficult to prove that the spectral radius of $\mI_{3,sc}^{-1}\mI_{2,sc}$ is less than $1$ in FISCORE. However, the experimental results presented in Section \ref{sec:experiment} imply that this algorithm converges in practice. 

\begin{remark}
Similar results hold for the non--negative score matching defined by $M_{sc+}(x;\theta)$. 
See Appendix \ref{sec:variance}.  
\end{remark}

\begin{remark}
See Appendix \ref{sec:variance} for variance estimators based on Theorems~\ref{thm:easy2},\,\ref{thm:easy} and \ref{thm:main}. 
\end{remark}

\section{EXTENSION TO MNAR CASE}\label{sec:extensions}

We discuss an extension to the case of missing not at random (MNAR). 
In general, the nonparametric identification condition does not hold in the MNAR case \citep{RobinsJM1997Taco}. However, assuming the existence of nonresponse instrument and parametric models, the parameter can be identified in some cases \citep{KimJiYoung2012Pfif,wang14}. We hereafter assume the existence of a nonresponse instrument so that the parameter can be identified. 
\begin{assumption}[\citet{wang14} ]
There exists nonresponse instrument $\bx_2$ s.t. $\bx=(\bx_1,\bx_2)$ and $\bx_2 \bigCI \delta \mid \bx_1$.
\end{assumption}

To estimate the parameter under MNAR data, FISCORE and FINCE can be still applied. First, we specify a propensity score model $\pi(\delta|x;\phi)$ for $\mathrm{Pr}(\delta|x)$. For the case of FISCORE, we want to solve the equation with respect to $\eta$: 
\begin{align}
\label{eq:ideal2}
    \mathrm{E}\left[\begin{pmatrix}
    Z_{sc}(\bx;\theta) \\
    \nabla_{\phi} \log \pi(\bm{\delta}|\bx;\phi)
    \end{pmatrix}
    |\bx_{\mathrm{obs}},\bm{\delta};\eta \right]=0,
\end{align}
where the expectation is taken under $t(\mathbf{x}_{\mathrm{mis}}|\mathbf{x}_{\mathrm{obs}},\bm{\delta};\eta)\propto p(\mathbf{x};\theta)\pi(\bm{\delta}|\mathbf{x};\phi)$, and $\eta=(\theta,\phi)$. 
Importantly, we must address the selection mechanism unlike in the MAR and MCAR cases, because $p(x_{\mathrm{mis}}|x_{\mathrm{obs}})=p(x_{\mathrm{mis}}|\delta,x_{\mathrm{obs}})$ does not hold. The difference is evident when we compare \eqref{eq:ideal2} with \eqref{eq:fiscore}. Owing to MNAR, the first modification is such that the selection mechanism $\mathrm{\pi}(\delta|x)$ appears when calculating the fractional weight: $w_{ik} \propto \tilde{p}(x_{i}^{*k};\hat{\theta}_{t})\pi(\delta_i|x_{i}^{*k};\hat{\phi}_{t})/b(x_{\mathrm{mis}}^{*k})$. The second modification is the score function of the propensity score model which is illustrated in \eqref{eq:ideal2}. 

In the case of FINCE, let $\zeta =(\tau^{\top},\phi^{\top})^{\top}$ and $Z_{nc}(\bm{\delta},\bx,\by;\zeta)$ be defined as an augmented estimating equation:
\begin{align*}
\begin{pmatrix}
Z_{nc1}(\bx;\tau)+Z_{nc2}(\by;\tau)\\
\nabla_{\phi} \log \pi(\bm{\delta}|\bx;\phi)
\end{pmatrix}.
\end{align*}
The algorithm is modified to solve the following equation with respect to $\zeta$:
\begin{align*}
   \mathrm{E}\left[\begin{pmatrix}
Z_{nc1}(\bx;\tau)+Z_{nc2}(\by;\tau) \\
\nabla_{\phi} \log \pi(\bm{\delta}|\bx;\phi)
\end{pmatrix}|\bx_{\mathrm{obs}},\bm{\delta};\zeta \right]=0. 
\end{align*}
Refer to Appendix \ref{sec:mnar} for the details. 

\section{SIMULATION RESULTS}
\label{sec:experiment}
Here, we confirm the validity of FINCE and FISCORE by simulation. 
We do not compare them with variational NCE because the latter does not take the MNAR case into account and does not provide a confidence interval. 

\subsection{Truncated Normal Distribution}

First, we consider covariance estimation for the two-dimensional truncated normal distribution:
\begin{align*}
\tilde{p}(x_1,x_2; \Sigma) = \exp \left( -\frac{1}{2} x^{\top} \Sigma^{-1} x \right),\Sigma=\begin{pmatrix} 2.0 & 0.3 \\ 0.3 & 2.0 \end{pmatrix},
\end{align*}
where $x_1,x_2 \geq 0$.
For $n$ samples of $x=(x_1,x_2)$, $x_1$ was always observed whereas missing data of $x_2$ was introduced with the following two mechanisms. 
The binary random variable $\delta$ is the missing indicator: $x_{2}$ was observed if and only if $\delta=1$.

\begin{itemize}
    \item MAR: $\mathrm{Pr}(\delta=1 \mid x)=1/[1+\exp\{-(x_1-0.9)/0.3\}]$ 
    \item MNAR: $\mathrm{Pr}(\delta=1 \mid x)=1/[1+\exp\{-5(x_2-0.9) \}]$ 
\end{itemize}

For data generation, we used the R-package \textit{mvtnorm} \citep{mvtnorm}.
In both cases, the overall missing rates were about 30\%.

We compared three estimators of $\Sigma$: original NCE based on complete data,  (original) FINCE, and FISCORE. 
Note that the noise and auxilliary distributions were set to truncated normal distributions with moment matching and $m=100$ imputations were used in FINCE and FISCORE. 

Table \ref{tab:exp1} presents the absolute bias and median squared error from $200$ simulations. 
It shows that NCE on complete data has significant bias in both the MAR and MNAR cases, which is expected from the theory of missing data \citep{LittleRoderickJ.A2002Sawm}. 
On the other hand, {FINCE} and {FISCORE} achieve better estimation accuracy with reduced bias by appropriately accounting for the missing data. 

\begin{table}[!]
    \centering
    \caption{The absolute bias and median square error}
      \label{tab:exp1}
    MAR \\
    \begin{tabular}{llccc} \toprule 
    $n$     & & NCE & FINCE & FISCORE   \\ \midrule
    500 & (bias) & 0.29    & 0.03   & 0.03   \\
       & (mse) & 0.040     & 0.024  & 0.021       \\
    1000 & (bias) & 0.25   & 0.02 &  0.02    \\
       & (mse)  & 0.032      & 0.015  &  0.011     \\  \bottomrule
    \end{tabular}
    \\ MNAR \\
    \begin{tabular}{llccc}\toprule
      $n$   & & NCE & FINCE & FISCORE \\ \midrule
    500 & (bias) &  0.33  & 0.18   &  0.12  \\ 
       & (mse) & 0.041     & 0.027 & 0.021       \\
    1000 & (bias) & 0.24   & 0.14 &  0.14     \\
       & (mse)  &  0.039  & 0.020 &    0.012  \\ \bottomrule
    \end{tabular}
    \vspace{-0.2cm}
\end{table}

In addition, we constructed $95$\% confidence intervals based on the variance estimators in Appendix \ref{sec:variance}. 
Table \ref{tab:exp4} shows their coverage probabilities: they are approximately equal to 95\% in both FINCE and FISCORE.

\begin{table}[h!]
    \centering
    \caption{Coverage probabilities: MAR setting }
    \begin{tabular}{ccc} \toprule
    $n$     & FINCE & FISCORE \\ \midrule
    500 &  94\%    &  89\%  \\
    1000 & 94\%     &  92\%  \\  \bottomrule
    \end{tabular}
        \label{tab:exp4}
        \vspace{-0.2cm}
\end{table}

\subsection{Truncated Gaussian Graphical Model}
Next, we consider estimation of the truncated Gaussian graphical model (GGM) considered in \cite{LinLina2016EoHG} with missing data.

Let $G=(V,E)$ be an undirected graph where $V=\{ 1,\cdots,d \}$. 
Then, a truncated GGM with graph $G$ is defined as the unnormalized model \eqref{nnm} with 
\begin{align}
\tilde{p}(x;\Sigma) = \exp \left( -\frac{1}{2} x^{\top} \Sigma^{-1} x \right) \quad (x \in \mathbb{R}_+^d), \label{tGGM}
\end{align}
where $\Sigma$ is a ${d \times d}$ positive definite matrix satisfying $(\Sigma^{-1})_{ij}=0$ for $(i,j) \not\in E$.
Similar to the original GGM \citep{LauritzenSteffenL1996Gm}, $X_i$ and $X_j$ are conditionally independent on the other variables $X_k \ (k \neq i,j)$ if $(i,j) \not\in E$.
Here, we estimate $G$ by using the confidence intervals of the entries of $\Sigma^{-1}$.

We generated $n=1000$ independent samples from a truncated GGM  with $d=10$ and $G$ provided in the top panel of Figure~\ref{fig_GGM}.
Namely, there are three clusters $(x_1,x_2,x_3),(x_4,x_5,x_6)$, and $(x_7,x_8,x_9)$ of three variables and one isolated variable $x_{10}$.
We set all diagonal entries of $\Sigma^{-1}$ to 1 and all nonzero off-diagonal entries of $\Sigma^{-1}$ to 0.5.
We introduced missing values on $x_3$, $x_6$ and $x_9$ by using the following MAR mechanism: for $k=1,2,3$, random vector $c_k \in \mathbb{R}^{10}$ was generated by $(c_k)_3=(c_k)_6=(c_k)_9=0$ and $(c_k)_j \sim {\rm N} (0,1) \ (j \neq 3,6,9)$ and then $x_{3k}$ was missed with probability $1/(3+\exp(c_k^{\top} x))$.
The proportion of missing data was approximately 60\%. 

Then, we fitted the truncated GGM by using FINCE and FISCORE with 100 imputations.
We used ${\rm N}(0,2)$ truncated to the positive orthant as the proposal distribution for missing entries.
In FINCE, we generated $n=1000$ noise samples from the product of the coordinate-wise exponential distributions with the same mean as the data.

We determined graph $G$ by collecting all edges $(i,j)$ for which the 95 \% confidence interval of $(\Sigma^{-1})_{ij}$ did not include zero.
Figure~\ref{fig_GGM} presents the result of one realization.

Table~\ref{tab:ggm} shows the proportions of falsely selected edges (false positive) and falsely unselected edges (false negative) in 100 realizations.
The coverage probabilities of the confidence intervals are approximately equal to 95\% in both FINCE and FISCORE.

\begin{figure}[h!]
\begin{minipage}{0.5\textwidth}
\begin{center}
truth\\
\vspace{0.1in}
\scalebox{0.7}{
\begin{tikzpicture}[every node/.style={circle,draw}]
    \node (A) at (0,0) {};
    \node (B) at (1,0) {};
    \node (C) at (0.5,0.5) {};
    \node (D) at (1,1) {};
    \node (E) at (2,1) {};
    \node (F) at (1.5,1.5) {};
    \node (G) at (2,0) {};
    \node (H) at (3,0) {};
    \node (I) at (2.5,0.5) {};
    \node (J) at (1.5,0.5) {};
    \foreach \u \v in {A/B,B/C,C/A,D/E,E/F,F/D,G/H,H/I,I/G}
        \draw (\u) -- (\v);
\end{tikzpicture}
}
\end{center}
\end{minipage}
\vspace{0.1in}
\\
\begin{minipage}{0.5\textwidth}
\begin{minipage}{0.45\textwidth}
\begin{center}
FINCE\\
\vspace{0.1in}
\scalebox{0.7}{
\begin{tikzpicture}[every node/.style={circle,draw}]
    \node (A) at (0,0) {};
    \node (B) at (1,0) {};
    \node (C) at (0.5,0.5) {};
    \node (D) at (1,1) {};
    \node (E) at (2,1) {};
    \node (F) at (1.5,1.5) {};
    \node (G) at (2,0) {};
    \node (H) at (3,0) {};
    \node (I) at (2.5,0.5) {};
    \node (J) at (1.5,0.5) {};
    \foreach \u \v in {B/C,D/E,D/F,E/G,G/I,H/I}
        \draw (\u) -- (\v);
\end{tikzpicture}
}
\end{center}
\end{minipage}
\begin{minipage}{0.45\textwidth}
\begin{center}
FISCORE\\
\vspace{0.1in}
\scalebox{0.7}{
\begin{tikzpicture}[every node/.style={circle,draw}]
    \node (A) at (0,0) {};
    \node (B) at (1,0) {};
    \node (C) at (0.5,0.5) {};
    \node (D) at (1,1) {};
    \node (E) at (2,1) {};
    \node (F) at (1.5,1.5) {};
    \node (G) at (2,0) {};
    \node (H) at (3,0) {};
    \node (I) at (2.5,0.5) {};
    \node (J) at (1.5,0.5) {};
    \foreach \u \v in {A/B,A/C,B/C,B/F,D/E,D/F,E/F,G/I,H/I}
        \draw (\u) -- (\v);
\end{tikzpicture}
}
\end{center}
\end{minipage}
\end{minipage}
\caption{True and selected graphs for the truncated GGM} 
\label{fig_GGM}
\vspace{-0.3cm}
\end{figure}
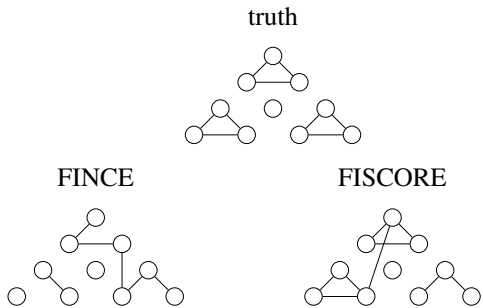

\begin{table}[h!]
    \centering
    \caption{Proportions of false positives (FP) and false negatives (FN) in edge selection of truncated GGM} 
    \begin{tabular}{llccc}\toprule 
         & FINCE & FISCORE \\ \midrule 
    FP &  10.5\%    &  6.4\%  \\
    FN &  12.6\%     &  23.3\%   \\  \bottomrule
    \end{tabular}
        \label{tab:ggm}
\vspace{-0.3cm}
\end{table}

\section{APPLICATION TO REAL DATA}
Many models in directional statistics have intractable normalization constants \citep{mardia99}.
Here, we consider estimation of the bivariate circular distribution proposed by \cite{singh02}, which is a probability distribution on two circular variables $x_1,x_2 \in [0,2\pi)$. 
It is an unnormalized model \eqref{nnm} with
\begin{align}
\tilde{p}(x_1,x_2 ; \theta) = & \exp (\kappa_1 \cos(x_1-\mu_1) + \kappa_2 \cos(x_2-\mu_2) \nonumber \\
&+ \lambda_{12} \sin(x_1-\mu_1) \sin(x_2-\mu_2)), \label{bvM}
\end{align}
where $\theta=(\kappa_1,\kappa_2,\mu_1,\mu_2,\lambda_{12})$. 
For identifiability, we imposed the parameter constraints $\kappa_1 \geq 0$, $\kappa_2 \geq 0$, $0 \leq \mu_1 <2\pi$ and $0 \leq \mu_2 <2\pi$.
Although \cite{mardia08} developed a method for estimating $\theta$ based on pseudo-likelihood, it is not applicable to missing data.
We applied FINCE to estimate $\theta$ from missing data.

We used the data on wind direction in Tokyo on 00:00 ($x_1$) and 12:00 ($x_2$) for each day in 2018\footnote{available on Japan Meteorological Agency website}.
Thus, the sample size is $n=365$.
The data were discretized into 16 directions, such as north-northeast.
Figure~\ref{fig:wind} presents a $16 \times 16$ histogram of raw data in gray scale.

\begin{figure}
    \centering
    \includegraphics[width=50mm]{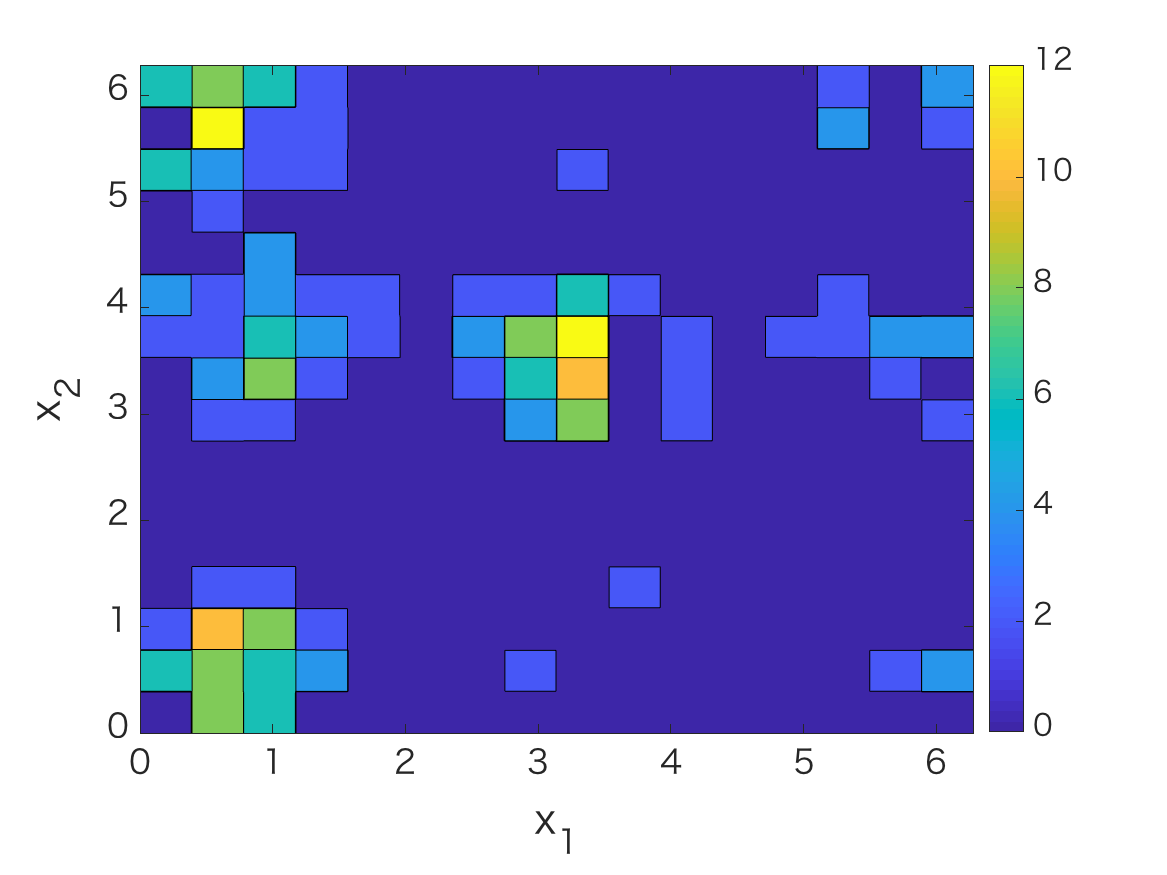} 
    \caption{Wind direction data}
    \label{fig:wind}
\vspace{-0.3cm}
\end{figure}

For each of the 365 samples, we introduced missing for $x_2$ with probability $1/(1+\exp(\cos(x_1)))$. Thus, the missing mechanism was MAR. Among 365 samples, $x_2$ was observed in 212 samples.

Then, we fit the model \eqref{bvM} by FINCE with 100 imputations and 1000 noise samples, where the noise distribution and proposal distribution for missing entries were set to uniform distributions on $[0,2\pi) \times [0,2\pi)$ and $[0,2\pi)$, respectively.
For comparison, we also fit the model \eqref{bvM} by NCE, in which 153 samples with $x_2$ missing were simply discarded, which is called the complete case analysis and known to have bias in MAR cases \citep{kimshao13}. 

Table~\ref{tab:bvM} presents the confidence intervals obtained by FINCE, NCE on complete cases, and NCE on the full data (365 samples) for reference. 
Whereas complete case analysis has a large bias, FINCE provides similar estimates to NCE on full data, with wider confidence intervals due to missing.
In particular, FINCE succeeds in detecting the 5\% significance of $\lambda_{12} \neq 0$, which implies that $x_1$ and $x_2$ are not independent.
Thus, FINCE enables statistical inference based on unnormalized models by properly  handling missing data.

\begin{table}[h!]
    \centering
    \caption{Confidence intervals for bivariate circular distribution \eqref{bvM}. (cc: complete case)}
    \begin{tabular}{cccc}\toprule 
         & FINCE & NCE (cc) & NCE (full) \\ \midrule
    $\kappa_1$ & [0.30,1.17] & [0.58,1.99] & [0.26,1.16] \\
    $\mu_1$ & [0.59,1.62] & [-0.18,1.31] & [0.69,1.58] \\
    $\kappa_2$ & [0.11,1.09] & [-0.07,0.96] & [0.16,0.84] \\
    $\mu_2$ & [4.08,5.01] & [3.46,5.70] & [4.09,4.85] \\
    $\lambda_{12}$ & [-2.20,-0.29] & [-1.96,0.94] & [-1.67,-0.48] \\ \bottomrule 
    \end{tabular}
        \label{tab:bvM}
    \vspace{-0.5cm}
\end{table}

\section{CONCLUSION}

We have proposed estimation methods for unnormalized models with missing data: FINCE and FISCORE. 
The proposed methods are computationally efficient, valid under general missing mechanisms, and enable statistical inference using the confidence intervals. Extending the recently developed statistically efficient estimators for unnormalized models \citep{UeharaMasatoshi2019Ueff} to missing data is an interesting future research. 

\newpage
\subsubsection*{Acknowledgements}
We would like to thank the anonymous reviewers for
their insightful comments and suggestions. 

Masatoshi Uehara was supported in part by MASASON Foundation. Takeru Matsuda was partially supported by JSPS KAKENHI Grant Numbers 16H06533 and 19K20220.

\bibliographystyle{chicago}
\bibliography{pfi}

\newpage 

\appendix

\onecolumn

\section{SUMMARY OF NOTATIONS}\label{sec:notation}

\begin{table}[!htbp]
    \centering
    \caption{Summary of notations}
    \begin{tabular}{l|l}
     $g(x)$    &  True density \\
     $a(x)$    & Auxiliary density \\
     $b(x)$ & Noise distribution \\
     $n$ & Sample size \\
     $\tilde{p}(x;\theta)$ & Unnormalized model \\
     $p(x;\theta)$ & Normalized model \\
     $q(x;\tau)$ & One-parameter extended model \\
     $x_{\mathrm{obs}}\,,x_{\mathrm{mis}}$ & Observed data and missing data \\
     $r(x)$ & $q(x;\tau)/a(x)$ or $\tilde{p}(x;\theta)/a(x)$ \\
     $\mathrm{Pr}(\delta|x,\phi)$ & Selection probability \\
     $\pi(\delta|x,\phi)$ & Propensity score model \\
     $\eta$ & $(\theta,\phi)$ \\
     $\zeta$ & $(\tau,\phi)$ \\
     $M_{sc}$ & Loss function of score matching \\
     $Z_{sc}$ & Estimating equation of score matching \\
     $M_{nc},M_{nc1},M_{nc2}$ & Loss function of NCE \\
     $Z_{nc},Z_{nc1},Z_{nc2}$ & Estimating equation of NCE \\
     $p(x;\theta)$ & Normalized model of $\tilde{p}(x;\theta)$ \\ 
     $t(x_{\mathrm{mis}};\eta)$ &  Posterior $p(x;\theta)\pi(\delta|x;\phi)$\\
    $\theta_{0}$ & True $\theta$ \\
    $x_{i}^{(*k)}$ & Imputed data \\
    $t(x)^{\otimes 2}$ & $t(x)t(x)^{\top}$ \\
    $\hat{\eta}_{p}$ & Initial estimator \\
     $\hat{\eta}_{sc}$ & Estimator by  FISCORE and MISCORE\\
     $\hat{\eta}_{nc,f}$ & Estimator by FINCE and MINCE \\
     $\mu$ & Baseline measure \\
     $c_{s}(x;\theta)$ & $\nabla_{x^s}\log \tilde{p}(x;\theta)$
    \end{tabular}
    \label{tab:my_label}
\end{table}

\section{PROOF}

To keep the clarity of the main points of this section, we will not specify regularity conditions. For details, see Chapter 5 in \cite{VaartA.W.vander1998As}. 

\begin{proof}[Proof of Theorem \ref{thm:easy} amd  \ref{thm:easy2}]
Direct calculation based on the original theory of M-estimator. 
\end{proof}

\begin{proof}[Proof of Theorem \ref{thm:main}]
First, we discuss the general derivation without using a specific form of $z_{sc}(\theta)$ so that it can be applied to NCE case. Then, we derived the specific formula for FISCORE. 

we have 
 \begin{align*}
 \bar{Z}_{sc}(\theta|\hat{\theta}_{p})=Z_{sc,\mathrm{obs}}(\theta)+\mathrm{E}[Z_{sc,\mathrm{mis}}|\bx_{\mathrm{obs}};\hat{\theta}_{p}],
 \end{align*}
 where $Z_{sc,\mathrm{mis}}=Z_{sc}(\theta)-Z_{sc,\mathrm{obs}}(\theta).$
By Taylor expansion, we have 
\begin{align*}
&\mathrm{E}[Z_{sc,\mathrm{mis}}|\bx_{\mathrm{obs}};\hat{\theta}_{p}] 
&= \mathrm{E}[Z_{sc,\mathrm{mis}}(\theta_0)|\bx_{\mathrm{obs}};\theta_{0}]+
\mathrm{E}[Z_{sc,\mathrm{mis}}(\theta)\nabla_{\theta^{\top}}\log p(\bx_{\mathrm{mis}}|\bx_{\mathrm{obs}};\theta)|\bx_{\mathrm{obs}};\theta_{0}]|_{\theta_0}(\hat{\theta}_{p}-\theta_{0})+\op.
\end{align*}

Therefore, 
 \begin{align}
 \label{eq:second}
 \bar{Z}_{sc}(\theta_{0}|\hat{\theta}_{p})   
 &=Z_{sc,\mathrm{obs}}(\theta_{0})-\mI_{2,sc}(\hat{\theta}_{p}-\theta_{0})+\op \nonumber \\
 &=-\mathcal{I}_{1,sc}(\hat{\theta}_{sc}-\theta_{0}) -\mathcal{I}_{2,sc}(\hat{\theta}_{p}-\theta_{0})+\op \\
 &=(-\mathcal{I}_{1,sc}-\mathcal{I}_{2,sc})(\hat{\theta}_{sc}-\theta_{0}) -\mathcal{I}_{2,sc}(\hat{\theta}_{p}-\hat{\theta}_{sc})+\op \nonumber,
 \end{align}
where 
\begin{align*}
    \mathcal{I}_{1,sc}&= \mathrm{E}[\nabla_{\theta^{\top}}Z_{sc,\mathrm{obs}}(\theta_{0})],\\
    \mathcal{I}_{2,sc}&=-\mathrm{E}[\mathrm{E}[Z_{sc,\mathrm{mis}}(\theta_{0})\nabla_{\theta^{\top}}\log p(\bx_{\mathrm{mis}}|\bx_{\mathrm{obs}};\theta_{0})]] \\
    &=-\mathrm{E}[\mathrm{E}[z_{sc,\mathrm{mis}}(\theta_{0})\nabla_{\theta^{\top}}\log p(x_{\mathrm{mis}}|x_{\mathrm{obs}};\theta_{0})]]  \\
    &=-\mathrm{E}[z_{sc,\mathrm{mis}}(\theta_{0})\nabla_{\theta^{\top}}\log p(x_{\mathrm{mis}}|x_{\mathrm{obs}})]  \\
    &=-\mathrm{E}[\mathrm{cov}[z_{sc,\mathrm{mis}}(\theta_{0}),\nabla_{\theta^{\top}}\log \tilde{p}(x_{\mathrm{mis}}|x_{\mathrm{obs}};\theta_{0})]|x_{\mathrm{obs}};\theta_{0}]. 
\end{align*}
From the first line to the second line \eqref{eq:second}, we used $\mathrm{E}[Z_{sc,\mathrm{mis}}(\theta_{0})|\bx_{\mathrm{obs}};\theta_{0}]=0$ and Theorem \ref{thm:easy}. 

In addition, since $\hat{\theta}_{sc,\infty}$ is the solution to $\bar{Z}_{sc}(\theta|\hat{\theta}_{p})$. Then,
\begin{align*}
0 &= \bar{Z}_{sc}(\hat{\theta}_{sc,\infty}|\hat{\theta}_{p})\\
&=\bar{Z}_{sc}(\theta_{0}|\hat{\theta}_{p})+\mathrm{E}[{\nabla_{\theta^{\top}}}Z_{sc}(\theta_0)](\hat{\theta}_{sc,\infty}-\theta_{0})+\op\\
&=\bar{Z}_{sc}(\theta_{0}|\hat{\theta}_{p})+\mathcal{I}_{3,sc}(\hat{\theta}_{sc,\infty}-\theta_{0})+\op,
\end{align*}
where 
\begin{align*}
    \mathcal{I}_{3,sc}= \mathrm{E}[\nabla_{\theta^{\top}}Z_{sc}(\theta_{0})].
\end{align*}
Therefore, we get
\begin{align*}
    (\hat{\theta}_{sc,\infty}-\theta_{0})
    &=-\mathcal{I}^{-1}_{3,sc}\{(-\mI_{1,sc}-\mI_{2,sc})(\hat{\theta}_{sc}-\theta_{0}) -\mathcal{I}_{2,sc}(\hat{\theta}_{p}-\hat{\theta}_{sc})\}+\op, \\
    &=(\hat{\theta}_{sc}-\theta_{0}) +\mathcal{I}^{-1}_{3,sc}\mI_{2,sc}(\hat{\theta}_{p}-\hat{\theta}_{sc})+\op.
\end{align*}
From the first line to the second line of the last equation, we used the relation $\mI_{3,sc}=\mI_{1,sc}+\mI_{2,sc}$. 
This is proved by 
\begin{align*}
\mI_{1,sc}+\mI_{2,sc}
&=\mathrm{E}[\nabla_{\theta^{\top}}(\mathrm{E}[Z_{sc}(\theta)|\bx_{\mathrm{obs}};\theta])]
-\mathrm{E}[\mathrm{E}[Z_{sc,\mathrm{mis}}(\theta_{0})\nabla_{\theta^{\top}}\log p(\bx_{\mathrm{mis}}|\bx_{\mathrm{obs}};\theta_{0})|\bx_{\mathrm{obs}};\theta_{0}]] \\
&=\mathrm{E}[\nabla_{\theta^{\top}}Z_{sc}(\theta_{0})]=\mI_{3,sc}.
\end{align*}

We go back to the special case of FISCORE. 

Noting $m_{sc}(\theta)=\sum_{s=1}^{d_x}0.5c_{s}^{2}(x)+\nabla_{x^{s}}(c_{s}(x))$,
the term $z_{sc}(\theta)$ is 
\begin{align*}
    z_{sc}(\theta)=\sum_{s=1}^{d}\left \{c_{s}(x)\nabla_{\theta}(c_{s}(x))+\nabla_{x^{s}}(\nabla_{\theta}c_{s}(x))\right \}.
\end{align*}
Then, we have 
\begin{align*}
&\mathrm{E}[\nabla_{\theta^{\top}}z_{sc,\mathrm{obs}}(\theta)]|_{\theta_0} =\mathrm{E}[\nabla_{\theta^{\top}}\{\mathrm{E}[z_{sc}(\theta)|x_{\mathrm{obs}};\theta]\}] |_{\theta_0}\\
=&\mathrm{E}[\nabla_{\theta^{\top}} z_{sc}(\theta)]|_{\theta_0}+\mathrm{E}[\mathrm{E}[z_{sc}(\theta)\{\nabla_{\theta^{\top}}\log p(x_{\mathrm{mis}}|x_{\mathrm{obs}};\theta)\}|x_{\mathrm{obs}};\theta_{0}]]|_{\theta_0} \\
=&\mathrm{E}[\nabla_{\theta^{\top}} z_{sc}(\theta)]|_{\theta_0}+\mathrm{E}[z_{sc}(\theta)\{\nabla_{\theta^{\top}}\log p(x_{\mathrm{mis}}|x_{\mathrm{obs}};\theta)\}]|_{\theta_0},
\end{align*}
where $\nabla_{\theta}\log p(x_{\mathrm{mis}}|x_{\mathrm{obs}};\theta)$ is 
\begin{align*}
\nabla_{\theta}\log \tilde{p}(x;\theta)- \mathrm{E}[\nabla_{\theta}\log  \tilde{p}(x;\theta)|x_{\mathrm{obs}};\theta]. 
\end{align*}
So, the above is equal to 
\begin{align*}
 \mathrm{E}[\nabla_{\theta^{\top}} z_{sc}(\theta)]|_{\theta_0}+\mathrm{E}[\mathrm{cov}[z_{sc}(\theta),\nabla_{\theta}\log \tilde{p}(x;\theta)|x_{\mathrm{obs}}]]|_{\theta_0}.
\end{align*}
In addition, 
\begin{align*}
\mathrm{E}[\nabla_{\theta^{\top}} z_{sc}(\theta)]|_{\theta_0}&= \mathrm{E}\left[\sum_{s=1}^{d}\left \{\nabla_{\theta}c_{s}(x)\nabla_{\theta^{\top}}c_{s}(x) +c_{s}(x)\nabla_{\theta\theta^{\top}}c_{s}(x)+\nabla_{x^{s}}(\nabla_{\theta\theta^{\top}}c_{s}(x))\right \}\right]|_{\theta_0}\\
&= \mathrm{E}\left[\sum_{s=1}^{d}\{\nabla_{\theta}c_{s}(x)\}^{\otimes 2}\right]|_{\theta_0}.
\end{align*}
From the second line to the third line, 
we used a partial integration trick, which is a core concept of score matching. 
\end{proof}

\begin{proof}[Proof of Corollary \ref{cor:clear}]

Clear from Theorem \ref{thm:finite}. 

\end{proof}

\begin{proof}[Proof of Corollary \ref{cor:nce}]

First, we calculate $\mJ_{1,nc}$. By noting the sampling mechanism of full data is a stratified sampling, this is calculated as follows: 
\begin{align*}
n^{-1}(\mathrm{var}_{q}[\mathrm{E}[z_{nc1}(x;\tau_0)|x_{\mathrm{obs}}]]+\mathrm{var}_{a}[z_{nc2}(y;\tau_0)]).
\end{align*}
Next, we calculate $\mI_{1,nc}$: 
\begin{align}
    \mI_{1,nc} &= \mathrm{E}[\nabla_{\tau^{\top}}Z_{nc,\mathrm{obs}}(\tau)]|_{\tau_{0}}=\mathrm{E}[\nabla_{\tau^{\top}}z_{nc,\mathrm{obs}}(\tau)]|_{\tau_{0}}
    =\mathrm{E}[\nabla_{\tau^{\top}}\{\mathrm{E}[z_{nc}(x,y;\tau)|x_{\mathrm{obs}};\tau]\}]|_{\tau_{0}} \nonumber \\
    &=\mathrm{E}[\nabla_{\tau^{\top}} z_{nc}(x,y;\tau)]|_{\tau_{0}}+\mathrm{E}[z_{nc}(x,y;\tau)\{\nabla_{\tau^{\top}}\log \bar{q}(x_{\mathrm{mis}}|x_{\mathrm{obs}};\tau) \}]|_{\tau_{0}} \\
    &= \mI_{3,nc}-\mI_{2,nc}
    \label{eq:exa}.
\end{align}
where
\begin{align*}
 \bar{q}(x_{\mathrm{mis}}|x_{\mathrm{obs}};\tau)=q(x_{\mathrm{mis}},x_{\mathrm{obs}};\tau)/\int q(x_{\mathrm{mis}},x_{\mathrm{obs}};\tau)\mu(\mathrm{d}x_{\mathrm{mis}}).
\end{align*}
By some algebra, the first term in \eqref{eq:exa} is 
\begin{align*}
    \mI_{3,nc}=\mathrm{E}[\nabla_{\tau^{\top}} z_{nc}(x,y;\tau)]|_{\tau_{0}}=\mathrm{E}\left[\frac{\nabla_{\tau}\log q(x;\tau_0)^{\otimes 2}}{1+r}\right]|_{\tau_{0}}.
\end{align*}
In addition, the second term in \eqref{eq:exa} is 
\begin{align*}
\mI_{2,nc}&=-\mathrm{E}[z_{nc}(x,y;\tau)\{\nabla_{\tau^{\top}}\log \bar{q}(x_{\mathrm{mis}}|x_{\mathrm{obs}};\tau) \}]|_{\tau_{0}}\\
&=-\mathrm{E}[\mathrm{E}[z_{nc1}(x;\tau)|x_{\mathrm{obs}}]\{\nabla_{\tau^{\top}}\log \bar{q}(x_{\mathrm{mis}}|x_{\mathrm{obs}};\tau) \}]|_{\tau_{0}} \\
&=-\mathrm{E}[\mathrm{cov}[z_{nc1}(x;\tau),\nabla_{\tau}\log q(x;\tau)|x_{\mathrm{obs}}]] \\
&=\mI_{3,nc}-\mathrm{E}\left[\mathrm{E}\left[\frac{\nabla_{\tau}\log q(x;\tau_0)}{1+r}|x_{\mathrm{obs}}\right]\mathrm{E}\left[\nabla_{\tau^{\top}}\log q(x;\tau_0)|x_{\mathrm{obs}}\right]\right]. 
\end{align*}
Therefore, adding the first and the second term in \eqref{eq:exa}, we get 
\begin{align*}
\mI_{1,nc}= \mathrm{E}\left[\mathrm{E}\left[\frac{\nabla_{\tau}\log q(x;\tau_0)}{1+r}|x_{\mathrm{obs}}\right]\mathrm{E}\left[\nabla_{\tau^{\top}}\log q(x;\tau_0)|x_{\mathrm{obs}}\right]\right].
\end{align*}
\end{proof}

\begin{proof}[Proof of Corollary \ref{cor:convergece}]
By some algebra, as in the proof of Corollary \ref{cor:nce}, we obtain
\begin{align*}
\mI_{1,nc} &=\mathrm{E}\left[\mathrm{E}\left[\nabla_{\tau}\log q(x;\tau_0)|x_{\mathrm{obs}}\right]^{\otimes 2}\right],\\
\mI_{3,nc} &=\mathrm{E}\left[\nabla_{\tau}\log q(x;\tau_0)^{\otimes 2}\right].
\end{align*}
So, noting that $\mI_{3,nc}$ is a positive definite matrix, and $\mI_{3,nc}$ and $\mI_{1,nc}$ are symmetric matrices, we can express $\mI_{3,nc}=RR^{\top}$ and $\mI_{1,nc}=R\Lambda R^{\top}$ using a nonsingular matrix $R$ \citep{RaoC.Radhakrishna2008LSIa}. Because $\mI_{3,nc}-\mI_{1,nc}$ is a positive matrix from Jensen's inequality, each element in $\Lambda$ is less than $1$. Then, we get
\begin{align*}
   \mI_{3,nc}^{-1}\mI_{2,nc}=\mI_{3,nc}^{-1}(\mI_{3,nc}-\mI_{1,nc})=R^{-1}(I-\Lambda)R.
\end{align*}
Finally, 
\begin{align*}
   (\mI_{3,nc}^{-1}\mI_{2,nc})^{j}=R^{-1}(I-\Lambda)^{j}R.
\end{align*}
Therefore, $\{\mI_{3,nc}^{-1}\mI_{2,nc}\}^{j}$ converges to zero as $j$ tends to infinity. 
\end{proof}

\section{Multiple missing patterns}
\label{sec:multiple}
We explain how to handle the case when the missing pattern is multiple. In the most general case, we have to introduce a missing pattern indicator $\delta$ that takes values in $0,1,…,2^K-1$, where $K$ is the dimension of $x$, for each sample. This is based on the fact that there are $2^{K}$ possible missing patterns for $x$. In the main manuscript, we have defined $\delta$ to take a value of $0$ or $1$ because there are only two missing patterns, i.e., one value is missing or not. Then, we have to introduce an importance distribution $b(x_{mis})$ separately for each missing pattern in the general case. In practice, we can simply select the importance distribution for each coordinate and take their products. Thus, the proposed methods can be applied to the general missing case. For example, in Algorithm 2, we impute missing values of each sample using the importance distribution of corresponding missing pattern. Then, W-step and M-step are essentially the same. Here is a concrete example. 

\begin{example}
Consider the case $x_{1,obs}= [0, \tNA, 1],\,x_{2,obs} = [\tNA, \tNA, 4],x_{3,obs} = [1,2,3]$. 

In this case, there are $3$ missing patterns: $(*,*,*), (*,\tNA,* ), (\tNA, \tNA, *)$, where * means an observed value and $\tNA$ means a missing value. Then, we introduce a missing indicator $\delta$ to take values in ${0,1,2}$. Namely, $\delta=0$ corresponds to $(*,*,*)$, $\delta=1$ corresponds to $(*,\tNA,* )$, and $\delta=2$ corresponds to $ (\tNA, \tNA, *)$. Then, for $x_{1,obs}=[0,\tNA,1]$, we have $\delta_1=1$. For $x_{2,\mathrm{obs}}= [\tNA, \tNA, 4]$, we have $\delta_2=2$. Since $x_3$ is observed without missing, then $\delta_3=0$. \eqref{eq:em2} is rewritten:
\begin{align*}
    \frac{1}{n}\sum_{i=1}^{n}\sum_{k=0}^{2}\mathrm{I(\delta_i=k)}\mathrm{E}[z_{nc}(x_i;\tau)\mid x_{i,\mathrm{obs}};\theta]. 
\end{align*}
\end{example}

\section{COMPARISON BETWEEN FINCE and VNCE}\label{sec:difference}

Here, we compare FINCE and variational NCE (VNCE) \citep{Rhodes}. From \eqref{eq:nce}, the difference between the estimator proposed in this paper and VNCE \citep{Rhodes} is clearly shown. Mainly, there are two differences: (1) VNCE attempts to maximize the observed likelihood directly, whereas FINCE attempts to solve the observed estimating equation, (2) VNCE assumes that the dimension of $a(x)$ is the same as the dimension of $x_{\mathrm{obs}}$, whereas FINCE assumes that the dimension of $a(x)$ is the same as the dimension of $x$. 

More specifically, an ideal loss function in VNCE is
\begin{align}
\label{eq:vnce}
    &\argmax_{s}J_{\text{VNCE}}(\tau,s(x_{\mathrm{mis}}))=J_{\text{VNCE}}(\tau,q(x_{\mathrm{mis}}|x_{\mathrm{obs}})) \nonumber \\
    =& \frac{1}{n}\sum_{i=1}^{n}\mathrm{E}\left[\log\left\{\frac{1}{1+\frac{q(x_{\mathrm{mis}}|x_{i,\mathrm{obs}})a(x_{i,\mathrm{obs}})}{q(x_{\mathrm{mis}},x_{i,\mathrm{obs}})}}\right \}|x_{i,\mathrm{obs}}\right]+
    \frac{1}{n}\sum_{j=1}^{n}\log\left \{\frac{a(y_j)}{a(y_j)+
    \mathrm{E}[q(y_{j},y_{j,mis})|y_{j}]}\right \}\nonumber \\
    =& \frac{1}{n}\sum_{i=1}^{n}\log\left\{\frac{q(x_{i,\mathrm{obs}})}{q(x_{i,\mathrm{obs}})+a(x_{i,\mathrm{obs}})}\right \}+\frac{1}{n}\sum_{j=1}^{n}\log\left \{\frac{a(y_j)}{a(y_j)+q(y_{j})}\right \}, 
\end{align}
where $q(x_{\mathrm{obs}})=\int q(x_{\mathrm{mis}},x_{\mathrm{obs}})\mu(\mathrm{d}x_{\mathrm{mis}})$. 
On the other hand, the loss function of our proposed estimator is \eqref{eq:nce}. In general, the efficiencies of the two loss function are not directly comparable. The following points highlight the comparison between two methods.
\begin{itemize}
    \item In terms of inferences, our proposed methods (FINCE,\,FISCORE) are superior to VNCE because it is difficult to achieve the upper bound in \eqref{eq:vnce} in VNCE. 
    \item Unless the family of variational distribution includes the true posterior, VNCE does not have consistency. On the other hand, FINCE has consistency and also asymptotic normality without requiring such conditions. 
    \item In terms of the scalability, VNCE is superior to the proposed methods because VNCE does not require any sampling methods. 
    \item FINCE can be applied even if the missing data mechanism is MAR or MNAR. However, VNCE cannot be directly applied when the missing mechanism is MAR or NMAR. 
\end{itemize}

\section{INFERENCE OF FISCORE WHEN $m$ IS FIXED}\label{sec:finite}

We consider an asymptotic result of FISCORE when $m$ is fixed. Actually, the estimating equation $Z_{sc,m}$ is not unbiased estimator for $\bar{Z}_{sc}$ because a self normalizing importance sampling is used rather than importance sampling \citep{mcbook}. 
This means that the derived estimator is theoretically not consistent; however, practically, a self normalized importance sampling is preferable to importance sampling because of its robustness. Here, we consider the case when the weight is defined as $w(x|x_{\mathrm{obs}})=p(x_{\mathrm{mis}}|x_{\mathrm{obs}};\theta_0)/b(x)$.  

As in the proof of Theorem \ref{thm:main}, we have 
\begin{align}
   &\hat{\theta}_{sc,m}-\theta_{0}=-\mI_{3,sc}^{-1}Z_{sc,m}(\theta_0|\hat{\theta}_{p})+\mathrm{o}_{p}(n^{-1/2})
 \label{eq:twoterms}
\end{align}
This term is decomposed into two terms: $-\mI_{3,sc}^{-1}\bar{Z}_{sc}(\theta_0|\hat{\theta}_{p})$ and $-\mI_{3,sc}^{-1}\{Z_{sc,m}(\theta_0|\hat{\theta}_{p})-\bar{Z}_{sc}(\theta_0|\hat{\theta}_{p})\}$. These two terms in \eqref{eq:twoterms} are independent. The first term is equal to $\hat{\theta}_{sc,\infty}-\theta_{0}$, of which the asymptotic property is shown in Theorem \ref{thm:main}. The second term converges to the normal distribution with mean $0$ and variance $\mI_{3,sc}^{-1}\mathrm{E}[\mathrm{var}_{b}\{Z_{sc,m}(\theta_0|\hat{\theta}_{p})\}] \mI_{3,sc}^{\top-1}$.

\begin{theorem}
\label{thm:finite}
When $\hat{\theta}_{p}=\hat{\theta}_{sc}$, the asymptotic variance of $\hat{\theta}_{sc,m}$ is equal to 
\begin{align*}
\mI_{1,sc}^{-1}\mJ_{1,sc}\mI_{1,sc}^{\top-1}+m^{-1}\mI_{3,sc}^{-1}\mJ_{2,sc}\mI_{3,sc}^{\top-1},
\end{align*}
where $w(x|x_{\mathrm{obs}})=p(x_{\mathrm{mis}}|x_{\mathrm{obs}};\theta_0)/b(x)$ and 
\begin{align*}
  \mathcal{J}_{2,sc}=n^{-1}\mathrm{E}[\mathrm{E}_{b(x_{\mathrm{mis}})}[ w^{2}(x)\{z_{sc}(\theta_0)^{\otimes 2}\}|x_{\mathrm{obs}}]]-n^{-1}\mathrm{E}[\mathrm{E}[z_{sc}(\theta_0) |x_{\mathrm{obs}}]^{\otimes 2}].
\end{align*}
\end{theorem}

\section{EXTENSION TO MULTIPLE IMPUTATION: MISCORE AND MINCE}\label{sec:MI}

MI was originally developed with Bayesian flavor \citep{rubin87,meng94}. In this paper, we consider frequentist MI rather than Bayesian MI \citep{TsiatisAnastasiosA.AnastasiosAthanasios2006Stam} to avoid the additional computation. In addition, it is shown that frequentist MI is asymptotically more efficient than Bayesian MI \citep{wang98,robins00}.

In MI, the crucial assumption is that the sample can be obtained from $p(x_{\mathrm{mis}}|x_{\mathrm{obs}};\theta)$. When the missing data mechanism is MAR, it is easy to sample from  $p(x_{\mathrm{mis}}|x_{\mathrm{obs}};\theta)$ using the MCMC based on \eqref{eq:mcmc}. The algorithm is described as in Algorithm \ref{al:em5}.
In this paper, this approach is referred to as MISCORE. MINCE is also defined similarly. Nevertheless, we do not recommend Algorithm \ref{al:em5} for the practical reason of its instability and computational burden. 

\begin{algorithm}
   \Repeat{$\hat{\tau}_{t}$ converges}
   {
   		W-step: Take a set of $m$ samples from $x_{\mathrm{mis}}^{*k}\sim p(x_{\mathrm{mis}}|x_{\mathrm{obs}};\hat{\theta}_{t})$ using MCMC for each $i$ \\
       M-step:  Update the solution to the following function with respect to $\theta$ as $\hat{\theta}_{t+1}$: 
        \begin{align*}
           \frac{1}{nm}\sum_{i=1}^{n}\sum_{k=1}^{m} m_{sc}(x_{i}^{*k};\theta). 
        \end{align*}
   }
\caption{MISCORE}
\label{al:em5}
\end{algorithm}

Dues to the challenges associated with Algorithm \ref{al:em5}, we recommend the following algorithm. This algorithm is similar to the one in \cite{LevineRichardA2001IotM}. In the original MISCORE, a set of samples is generated at every step. This requires tremendous computational cost and causes instability. In Algorithm \ref{al:em6}, by constructing a $\sqrt{n}$--consistent estimator based on FISCORE at each step and updating by MISCORE one time, this limitation is overcome. 

\begin{algorithm}[H]
\label{al:em6}
   \Repeat{$\hat{\tau}_{t}$ converges}{
   	Do W-step and M-step in Algorithm \ref{al:em3} (FISCORE) 
  }
 Do W-step and M-step in Algorithm \ref{al:em5} (MISCORE)
\caption{One step MISCORE}
\end{algorithm}

Table \ref{tab:exp5} illustrates the experimental result. We generated a set of $50$ samples for each $i$ using MCMC in the last step. Compared with FINCE and FISCORE, the performance of one step MISCORE is worse. Perhaps, more step is needed. 

\begin{table}[!]
    \centering
    \caption{Monte Carlo median square error and bias}
        \label{tab:exp5}
    Setting 1  \\
    \begin{tabular}{llcc} \toprule
    n     &  & MINCE & MISCORE \\ \midrule
    500 & (bias) &  0.15  &   0.10    \\
       & (mse) &    0.037  &  0.024 \\
    1000 & (bias) & 0.13   & 0.12   \\
       & (mse)  &  0.030 &  0.020   \\ \bottomrule  
    \end{tabular}
 
\end{table}

The asymptotic property is obtained as follows. 
\begin{corollary}
\label{cor:miscore}
When $\hat{\theta}_{p}=\hat{\theta}_{sc}$ and $m$ is fixed, the asymptotic variance of $\hat{\theta}_{sc,\infty}$ is equal to 
\begin{align*}
\mI_{1,sc}^{-1}\mJ_{1,sc}\mI_{1,sc}^{\top-1}+m^{-1}\mI_{3,sc}^{-1}\mJ_{2,sc}\mI_{3,sc}^{\top -1},
\end{align*}
where 
\begin{align*}
  \mathcal{J}_{2,sc}=n^{-1}\{\mathrm{E}[z_{sc}(\theta_0)^{\otimes 2}]-\mathrm{E}[\mathrm{E}[z_{sc}(\theta_0) |x_{\mathrm{obs}}]^{\otimes 2}]\},
\end{align*}
and other terms are the same as in Theorem \ref{thm:main}.
\end{corollary}

\begin{proof}[Proof of Corollary \ref{cor:miscore}]
We just replace $b(x_{\mathrm{mis}})$ with $p(x_{\mathrm{mis}}|x_{\mathrm{obs}};\theta_0)$ in Theorem \ref{thm:finite}. 
\end{proof}

Finally, there are two things to note about MISCORE and MINCE. When the missing data mechanism is MNAR, we have to sample from $\tilde{p}(x_{\mathrm{mis}}|x_{\mathrm{obs}},\delta;\eta)\propto \tilde{p}(x_{\mathrm{mis}}|x_{\mathrm{obs}};\theta)\pi(\delta|x_{\mathrm{mis}},x_{\mathrm{obs}};\phi)$. In this case, the distribution becomes a doubly-intractable distribution \citep{MllerJ.2006AeMc, MurrayIain2012Mfdd}, and it is generally difficult to sample. Secondly, when we use a Bayesian multiple imputation assuming the prior distribution $\rho(\theta)$, even if the missing mechanism is MAR, we have to sample from $\tilde{p}(x_{\mathrm{mis}},\theta|x_{\mathrm{obs}}) \propto \tilde{p}(x_{\mathrm{mis}},x_{\mathrm{obs}};\theta)\rho(\theta)$. Often, data augmentation is utilized for this purpose \citep{tanner87}. However, even if the data augmentation is applied, we still have to deal with doubly--intractable distributions to calculate $\mathrm{Pr}(\theta|x)\propto \rho(\theta)p(x;\theta)$.

\section{EXTENSION TO CONTRASTIVE DIVERGENCE METHODS}\label{sec:CD}

Although there are several variations of contrastive divergence methods \citep{younes,TielemanTijmen2008TrBm}, the basic idea is that $\theta$ is updated by adding the gradient of log-likelihood $\log p(\bx;\theta)$ with respect to $\theta$:
\begin{align*}
    \frac{1}{n}\sum_{i=1}^{n}\nabla_{\theta} \log \tilde{p}(x_{i};\theta)-\mathrm{E}_{ p(x;\theta)}[\nabla_{\theta}\log \tilde{p}(x;\theta)], 
\end{align*}
multiplying some learning rate. When some data is not observed, the expected gradient becomes 
\begin{align*}
    \frac{1}{n}\sum_{i=1}^{n}\mathrm{E}[\nabla_{\theta} \log \tilde{p}(x_{i};\theta)|x_{i,\mathrm{obs}};\theta]-\mathrm{E}[\nabla_{\theta}\log \tilde{p}(x;\theta)]. 
\end{align*}
The expectation of the first term is taken under $p(x_{\mathrm{mis}}|x_{\mathrm{obs}};\theta)$. It is possible to sample from MCMC like \eqref{eq:mcmc} without involving  doubly-intractable distributions \citep{MllerJ.2006AeMc}. 
Therefore, the gradient is approximated as 
\begin{align*}
    \frac{1}{nm}\sum_{i=1}^{n}\sum_{k=1}^{m}\nabla_{\theta} \log \tilde{p}(x_{i}^{*k};\theta)-\frac{1}{n} \sum_{j=1}^{n}\nabla_{\theta}\log \tilde{p}(y_j;\theta),
\end{align*}
where $x_{i}^{*k}\sim p(x_{\mathrm{mis}}|x_{i,\mathrm{obs}};\theta) $ and $y_j\sim p(y;\theta)$. 
We refer the updating method using the above gradient as MICD. 

We can still use a FI approach for the approximation. By introducing an auxiliary distribution with a density $b(x)$, the gradient is approximated as  
\begin{align*}
    \frac{1}{n}\sum_{i=1}^{n}\sum_{k=1}^{m}w_{ik}\nabla_{\theta}\log \tilde{p}(x_{i}^{*k};\theta)-\frac{1}{n} \sum_{j=1}^{n}\nabla_{\theta}\log \tilde{p}(y_j;\theta). 
\end{align*}
where $x_{i}^{*k}\sim b(x),w_{ik}\propto \tilde{p}(x_{i}^{*k};\theta)/b(x_{i}^{*k}),y_j\sim p(y;\theta)$. We refer this approach to FICD. 

Furthermore, by introducing a noise distribution with a density $a(y)$ to prevent using MCMC totally, the gradient is approximated as  
\begin{align*}
    \frac{1}{n}\sum_{i=1}^{n}\sum_{k=1}^{m}w_{ik}\nabla_{\theta}\log \tilde{p}(x_{i}^{*k};\theta)-\frac{1}{n} \sum_{j=1}^{n}r_{j}\nabla_{\theta}\log \tilde{p}(y_j;\theta),
\end{align*}
where $x_{i}^{*k}\sim b(x),\,w_{ik}\propto \tilde{p}(x_{i}^{*k};\theta)/b(x_{i}^{*k}),\,y_j\sim a(y)$, and  $r_j\propto \tilde{p}(y_j;\theta)/a(y_j)$. 
In this case, the gradient is essentially equivalent to the loss function of FINCE when $f(x)=x\log x$ by profiling-out $c$.

\section{DETAILED ALGORITHM OF FINCE WITH MNAR DATA}\label{sec:mnar}

The algorithm is described as in Algorithm \ref{al:em_mnar}. 

\begin{algorithm}[!h]
Initialize $t=0$, $\hat{\zeta}_{0}=(\hat{c}_0,\hat{\theta}_0,\hat{\phi}_0)$ \\
Take $n$ samples $\{y_{j}\}_{j=1}^{n}$ from $a(y)$.\\
For $i$ with $\delta_i=0$, take $m$ samples $\{x_{i,\mathrm{mis}}^{*k}\}_{k=1}^{m}$ from $b(x)$. \\
For $i$ with $\delta_i=1$, set $m$ samples $\{x_{i}^{*k}\}_{k=1}^{m}$ to $x_{i}^{*k}=x_{i}$
\\
   \Repeat{$\hat{\tau}_{t}$ converges}
   {
   		W-Step:    \\ 
   		For $i$ with $\delta_i=0$;
   		$w_{ik} \propto q(x_{i}^{*k};\hat{\tau}_{t})\pi(\delta_i|x_{i}^{*k};\hat{\phi}_{t})/b(x^{*k}_{i,\mathrm{mis}})$, \\
          For $i$ with $\delta_i=1$; $w_{ik}=1/m$. 
        \\
        M-step: Solve the following equation for $\hat{\zeta}_{t+1}$ w.r.t $\zeta$:
        \begin{align*}
        &\frac{1}{n}\sum_{i=1}^{n}\sum_{k=1}^{m}w_{ik}z_{nc1}(x_{i}^{*k};\tau)+Z_{nc2}(\by;\zeta)=0, \\
        & \frac{1}{n}\sum_{i=1}^{n}\sum_{k=1}^{m} \nabla_{\phi}w_{ik} \log \pi(\delta_i|x_{i}^{*k};\phi)=0.
         \end{align*}
         $t = t+1$
   }
\caption{FINCE witn MNAR data}
\label{al:em_mnar}
\end{algorithm}

\section{VARIANCE ESTIMATORS OF FISCORE AND FINCE}\label{sec:variance}

\subsection{FISCORE}
The variance estimator of FISCORE in the case of Corollary is defined as follows:
$\hat{\mI}_{1,sc}^{-1}\hat{\mJ}_{1,sc}\hat{\mI}_{1,sc}^{\top -1}|_{\hat{\theta}}$, where
\begin{align*}
\hat{\mI}_{1,sc} &=\frac{1}{n}\sum_{i=1}^{n}\{\hat{\mI}_{1,sc1}(x_{i,\mathrm{obs}})+\hat{\mI}_{1,sc2}(x_{i,\mathrm{obs}})-\hat{\mI}_{1,sc3}(x_{i,\mathrm{obs}}) \},\\
\hat{\mI}_{1,sc1}(x_{i,\mathrm{obs}}) &= \sum_{k=1}^{m}w(x_{i}^{*k};\theta)\left(\sum_{s=1}^{d}\nabla_{\theta}c_{s}(x_{i}^{*k})\right)^{\otimes 2},\\
\hat{\mI}_{1,sc2}(x_{i,\mathrm{obs}})&=\sum_{k=1}^{m}w(x_{i}^{*k};\theta)z_{sc}(x_{i}^{*k})\nabla_{\theta^{\top}}\log \tilde{p}(x_{i}^{*k};\theta), \\
\hat{\mI}_{1,sc3}(x_{i,\mathrm{obs}}) &=\left(\sum_{k=1}^{m}w(x_{i}^{*k};\theta)z_{sc}(x_{i}^{*k})\right)\left(\sum_{k=1}^{m}w(x_{i}^{*k};\theta)\nabla_{\theta^{\top}}\log \tilde{p}(x_{i}^{*k};\theta)\right), \\
\hat{\mJ}_{1,sc} &=  \frac{1}{n(n-1)}\sum_{i=1}^{n}\left(\sum_{k=1}^{m}w(x_{i}^{*k};\theta)z_{sc}(x_{i}^{*k};\theta)-\bar{z} \right)^{\otimes 2},\\
\bar{z}&=\frac{1}{n}\sum_{i=1}^{n}\sum_{k=1}^{m}w(x_{i}^{*k};\theta)z_{sc}(x_{i}^{*k};\theta),\\
z_{sc}(\theta)&=\sum_{s=1}^{d}\left\{c_{s}(x)\nabla_{\theta}(c_{s}(x))+\nabla_{x^{s}}(\nabla_{\theta}c_{s}(x))\right\},\,c_{s}(x;\theta)=\nabla_{x^s}\log \tilde{p}(x;\theta).
\end{align*}

Next, consider an loss function and a variance estimator in truncated exponential family cases \citep{HyvarinenAapo2007Seos}. 
Assume that $\tilde{p}(x;\theta)$ is given by 
\begin{align*}
    \log \tilde{p}(x;\theta)=\sum_{k=1}^{d}\theta_k F_k(x).
\end{align*}
Let us denote two matrices: $d_{\theta}\times d$ matrix $K_1(x)$ with elements $\nabla_{x^b}F_{a}\,(1\leq a \leq d_{\theta},1\leq b\leq d)$ and $d_{\theta}\times 1$ matrix, $K_{i,2}(x)$ with elements $\nabla\nabla_{x^i}F_{a}\,(1\leq a \leq d_x)$. 
The loss function is written as $n^{-1}\sum_{i=1}^{n}z_{sc,t}(x_i;\theta)$. 
\begin{align*}
    z_{sc,t}(x;\theta)=0.5\theta^{\top}K_1(x)K_1(x)^{\top} \theta+\theta^{\top}\sum_{i=1}^{d_x}K_{i,2}(x). 
\end{align*}
The variance estimator is obtained almost in the same by replacing $z_{sc}(x)$ with $z_{sc,t}(x)$. The only modification is 
\begin{align*}
\hat{\mI}_{1,sc1}(x_{i,\mathrm{obs}}) = \sum_{k=1}^{m}w(x_{i}^{*k};\theta)K_1(x_{i}^{*k})K_1(x_{i}^{*k})^{\top}.
\end{align*}

\subsection{FINCE}
The variance estimator of FINCE in the case of Corollary \ref{cor:nce} is defined as follows: $\hat{\mI}_{1,nc}^{-1}\hat{\mJ}_{1,nc}\hat{\mI}_{1,nc}^{\top{-1}}|_{\hat{\tau}}$, where 
\begin{align*}
\hat{\mI}_{1,nc}&=\frac{1}{n}\sum_{i=1}^{n}\left(\sum_{k=1}^{m}w(x_{i}^{*k};\tau) \frac{\nabla_{\tau}\log q(x_{i}^{*k};\tau)}{1+q(x_{i}^{*k})/a(x_{i}^{*k})}\right)\left(\sum_{k=1}^{m} w(x_{i}^{*k};\tau) \nabla_{\tau^{\top}}\log q(x_{i}^{*k};\tau)\right),\\
\hat{\mJ}_{1,nc}&=\hat{\mJ}_{1,nc1}+\hat{\mJ}_{1,nc2}, \\
\hat{\mJ}_{1,nc1} &=  \frac{1}{n(n-1)}\sum_{i=1}^{n}\left(\sum_{k=1}^{m}w(x_{i}^{*k};\tau)z_{nc1}(x_{i}^{*k};\tau)-\bar{z}_{nc1} \right)^{\otimes 2},\\
\hat{\mJ}_{1,nc2} &=  \frac{1}{n(n-1)}\sum_{i=1}^{n}\left(z_{nc2}(y_{i};\tau)-\bar{z}_{nc2}\right)^{\otimes 2},\\
\bar{z}_{nc1}&=\frac{1}{n}\sum_{i=1}^{n}\sum_{k=1}^{m}w(x_{i}^{*k};\tau)z_{nc1}(x_{i}^{*k};\tau),\,\bar{z}_{nc2}=\frac{1}{n}\sum_{j=1}^{n}z_{nc2}(y_{j};\tau),\\
z_{nc1}(\tau)&=-\frac{\nabla_{\tau}\log q(x;\tau)}{1+r},\,z_{nc2}(\tau)=\frac{r\nabla_{\tau}\log q(x;\tau)}{1+r}.
\end{align*}
\end{document}